\documentclass[letterpaper, 10 pt, journal, twoside]{IEEEtran}

\usepackage{xcolor}
\usepackage{amsmath,amssymb,amsfonts,color,nicefrac,amsthm}
\usepackage{bbm}
\usepackage{graphicx}
\usepackage{algorithm}
\usepackage{algpseudocode}
\usepackage{comment}
\usepackage{caption}
\usepackage{subcaption}
\usepackage{cite}

\DeclareMathOperator*{\argmax}{arg\,max}
\DeclareMathOperator*{\argmin}{arg\,min}

\DeclareSymbolFont{matha}{OML}{txmi}{m}{it}
\DeclareMathSymbol{\varv}{\mathord}{matha}{118}
\renewcommand{\Re}{\mathbb{R}}

\let\emptyset\varnothing
\newtheorem{proposition}{Proposition}

\usepackage{fancyhdr}
\fancypagestyle{firstpage}{
    \fancyhf{}

    \fancyfoot[C]{\footnotesize © 2025 IEEE. Personal use of this material is permitted.  Permission from IEEE must be obtained for all other uses, in any current or future media, including reprinting/republishing this material for advertising or promotional purposes, creating new collective works, for resale or redistribution to servers or lists, or reuse of any copyrighted component of this work in other works.}
}

\begin{document}

\title{HCOA*: Hierarchical Class-ordered A* \\for Navigation in Semantic Environments}

\author{Evangelos Psomiadis and Panagiotis Tsiotras
\thanks{Manuscript received March 31, 2025; Revised June 23, 2025; Accepted August 03, 2025.}
\thanks{This paper was recommended for publication by Editor Aniket Bera upon evaluation of the Associate Editor and Reviewers' comments.
This work was supported by ARL award DCIST CRA W911NF-17-2-0181 and ONR award N00014-23-1-2304.}
\thanks{E. Psomiadis, and P. Tsiotras are with the D. Guggenheim School of Aerospace Engineering, Georgia Institute of Technology, Atlanta, GA, 30332-0150, USA. Email:
{\tt\footnotesize \{epsomiadis3,tsiotras\}@gatech.edu}}%
\thanks{Digital Object Identifier (DOI): see top of this page.}
}

\markboth{IEEE Robotics and Automation Letters. Preprint Version. Accepted August, 2025}
{Psomiadis \MakeLowercase{\textit{et al.}}: HCOA*: Hierarchical Class-ordered A* for Navigation in Semantic Environments} 

\maketitle

\thispagestyle{firstpage}

\begin{abstract}
This paper addresses the problem of robot navigation in mixed geometric/semantic 3D environments. 
Given a hierarchical representation of the environment, the objective is to navigate from a start position to a goal, while satisfying task-specific safety constraints and minimizing computational cost. 
We introduce Hierarchical Class-ordered A* (HCOA*), an algorithm that leverages the environment's hierarchy for efficient and safe path-planning in mixed geometric/semantic graphs.
We use a total order over the semantic classes and prove theoretical performance guarantees for the algorithm.
We propose three approaches for higher-layer node classification based on the semantics of the lowest layer: a Graph Neural Network method, a k-Nearest Neighbors method, and a Majority-Class method. 
We evaluate HCOA* in simulations on two 3D Scene Graphs, comparing it to the state-of-the-art and assessing the performance of each classification approach.
Results show that HCOA* reduces the computational time of navigation by up to 50\%, while maintaining near-optimal performance across a wide range of scenarios.
\end{abstract}


\IEEEpeerreviewmaketitle

\section{Introduction}
\IEEEPARstart{A}{s} robotic sensing technologies advance, enabling robots to perceive vast and diverse information, two fundamental questions arise:
\textit{What information from this extensive data stream is most important for a given task? and, second, how can the robot effectively utilize this information for decision-making?}
Hierarchical semantic environment representations, such as 3D Scene Graphs (3DSGs) \cite{armeni_iccv19, hughes2022hydra, hughes2024foundations}, provide rich and structured abstractions that mirror human-like reasoning, thus facilitating the selection and organization of information.

Previous research in hierarchical path-planning has primarily addressed the first question \cite{botea2004, Fernandez1998,  Warren1993, kremer2023snav}. 
In \cite{botea2004} the authors introduce Hierarchical Path-Finding A*, a hierarchical A* \cite{Hart1968_Astar} variant for grid-based maps.
Their approach partitions the map into clusters with designated entrance points, which are used for high-level path-planning.
The authors in \cite{Fernandez1998} propose a hierarchical graph search algorithm for graphs with edge weights represented as intervals.
In \cite{kremer2023snav}, the authors introduce S-Nav, a hierarchical path-planning algorithm for navigation in Situational Graphs, a multi-layer graph representation of indoor environments built on top of 3DSGs.

Semantic path-planning has focused on the second question above by incorporating semantics into the decision-making process.
In \cite{serdel2023_smana}, the authors propose SMaNa, a semantic-aware mapping and navigation framework that employs A* with edge weights computed using node distances and semantic labels. 
To avoid reliance on manually crafted, user-defined cost functions, several works adopted the concept of total order over semantic classes.
In \cite{Wooden2006}, a weighted function combining edge cost and semantic class is introduced to determine optimal paths while respecting a total semantic order. 
However, this approach requires global graph properties, making it computationally expensive. 
To address this limitation, the authors in \cite{Lim2020} propose Class-ordered A* (COA*), an extension of traditional A* that incorporates semantic information, based on a total order. 
The algorithm is shown to be both complete and optimal under the defined total order and is validated in various planning scenarios. 
This work was later extended to dynamic environments in \cite{Lim2021COLPA*}, where LPA* \cite{KOENIG200493} was adapted for online replanning in dynamic maps with semantic information.

Building on the ideas introduced in \cite{Fernandez1998} and \cite{Lim2020}, this paper proposes a hierarchical class-ordered path-planning algorithm that uses a total order on semantic classes. 
The notion of total order is applicable to a wide range of scenarios, including navigation in partially known environments \cite{Lim2020} and safety-critical applications. 
In this work, we focus on encoding a set of safety constraints as relationships between semantic categories, which are implicitly enforced during path-planning.

\begin{figure}[tb]
    \centering        
    {\includegraphics[width=1\linewidth]{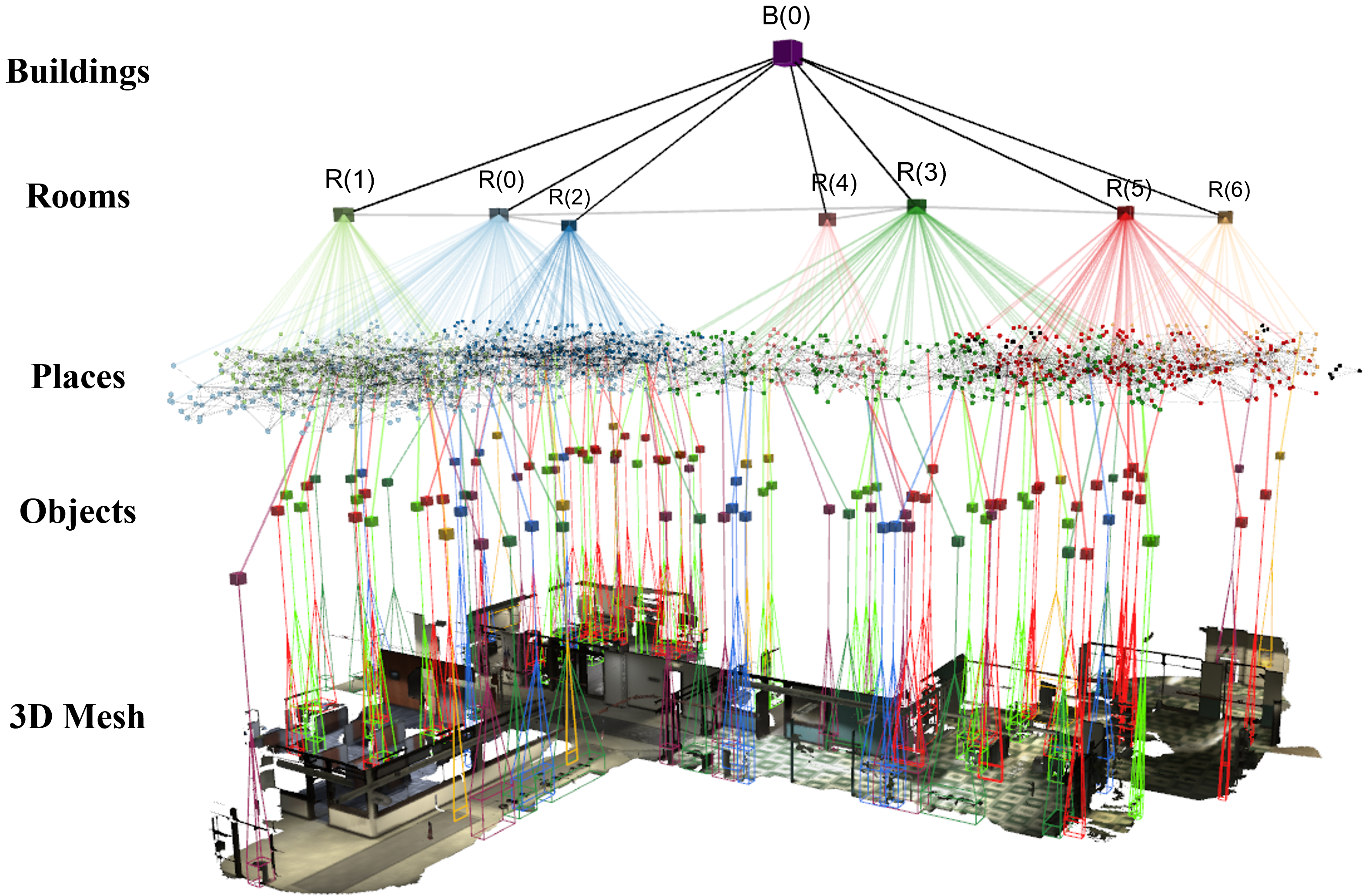}}
    \caption{3D Scene Graph generated from the uHumans2 office scene dataset \cite{Rosinol21ijrr-Kimera} using Hydra \cite{hughes2022hydra}. The graph comprises five layers, as shown in the figure. Room nodes are denoted as R($\cdot$), while building nodes are represented as B($\cdot$).} 
    \label{fig:3dsg}
\end{figure}

Navigation within hierarchical semantic environments has also been explored in prior work. 
In \cite{Ngom2024_Intellimove}, the authors present IntelliMove, a hierarchical semantic map framework along with a semantic path-planner. 
The edge weights in the graph encode spatial relationships, such as distances between rooms or objects, and Dijkstra’s algorithm is used to compute optimal paths.
Furthermore,~\cite{ray2024tamp} addresses Task and Motion Planning in 3DSGs using a three-level hierarchical planner consisting of a discrete task planner, a navigation planner for path-planning, and a low-level planner. 
Their approach focuses on optimizing task planning to reduce computational cost, while path-planning operates on the unpruned 3DSG.

Learning-based methods have also been explored to facilitate task execution in a 3DSG. 
In~\cite{Kurenkov2020SemanticAG}, the authors propose HMS, a neural network-based approach for object search in a 3DSG, leveraging the 3DSG's hierarchy. 
In \cite{Talak2021NeuralTree} and \cite{hughes2024foundations}, the authors propose the Neural Tree, a Graph Neural Network (GNN) architecture designed for node classification. 
Although effective in classifying higher-layer nodes in 3DSGs, this approach requires a tree decomposition of the input graph, which increases computational time.

\textbf{Main Contributions:}
Our approach integrates task semantics directly into the path-planning process within a unified algorithm. 
By leveraging the environment's hierarchy, we reduce computational resource demands while ensuring efficient and safe navigation. 
Our key contributions are as follows: 
\begin{enumerate} 
    \item We introduce Hierarchical Class-ordered A*, a novel hierarchical semantic path-planning algorithm for navigation in hierarchical semantic environments.
    \item We propose three methods for node classification on the higher layers of the environment's hierarchy: a Majority-Class (MC) method, a k-Nearest Neighbors (kNN) method, and a GNN method.
    \item We validate our approach on two publicly available 3DSG datasets, demonstrating its effectiveness in real-world scenarios.\footnote{The code is available at https://github.com/epsomiadis3/HCOA-star.}
\end{enumerate}

\section{Preliminaries} \label{sec:preliminaries}

We model the environment as a hierarchical semantic graph (HSG).
Specifically, let $G = (V,E,\mathcal{K})$ be a graph with $n$ layers, where $V$ is the set of nodes, $E \subseteq V \times V$ is the set of edges, and $\mathcal{K} = {1, \ldots K}$ represents a set of semantic classes ordered in decreasing priority.
This priority is characterized by task-specific requirements (e.g., avoid certain areas or objects, or prefer known over unknown regions).
We denote the set of layers as $L$, where $\ell=0$ is the lowest abstraction (e.g., sensor) layer and $\ell=n$ is the highest abstraction layer (root). 
Each layer $\ell$ forms a connected weighted subgraph $G^\ell = (V^\ell,E^\ell, \mathcal{K}) \subset G$, with an associated weight function $w^\ell: E^\ell \to \Re^+$, which corresponds to Euclidean distance in this work.
We assume that each node in layer $\ell-1$, for $\ell = 1, \ldots n$, is connected to a single node in layer $\ell$, which we refer to as its parent node.
More generally, we define an ancestor as a node's parent or any higher-layer predecessor in the hierarchy.
We define the projection function $p: V \times L \to V$ that maps each node to its corresponding ancestor in layer $\ell$.

Each node in $\ell = 0$ is assigned a semantic class based on the given perception data and the specific task using the function $\phi_V^0: V^0 \to \mathcal{K}$. 
We extend the labeling function $\phi_V^0$ to layers $\ell \neq 0$.
Details on the computation of $\phi_V^\ell$ for $\ell \neq 0$ are given in 
Section~\ref{sec:classification}.
Additionally, we define the edge classification function $\phi_E^\ell: E^\ell \to \mathcal{K}$ by $\phi_E^\ell(e) = \max(\phi_V^\ell(v), \phi_V^\ell(u))$ for $e = (v, u)$, thereby overestimating the edge class.

The environment model reflects the structure of 3DSGs, as shown in Figure~\ref{fig:3dsg}. 
For an overview of the advantages of this perceptual model and its construction, we refer the interested reader to \cite{hughes2022hydra}.
In our framework, in addition to the spatial layers, we also utilize the object layer of the 3DSG to encode safety constraints and to infer the nodes' semantic classes, enabling more context-aware and task-relevant navigation.

Let $\Pi(v_s^\ell, v_g^\ell)$ denote the set of all acyclic paths in $G^\ell$ from $v_s^\ell \in V^\ell$ to $v_g^\ell \in V^\ell$, and let $\Pi_k$ be the subset of all paths in $\Pi$ for which the least favorable (i.e., highest) edge class is exactly $k$.  
Formally,
\begin{equation}\label{eq:Pi_k_set}
\Pi_k = \{\pi \in \Pi : N(\pi,k) > 0; \: N(\pi,k') = 0, \forall k'>k\},
\end{equation}
where $N(\pi,k) = |\{e \in \pi: \phi_E^\ell(e) = k\}|$ is the number of edges of class $k$ in path $\pi$.
We further define $\Pi_k^i \subseteq \Pi_k$ consisting of all paths in $\Pi_k$ that contain exactly $i$ edges of class $k$, that is $\Pi_k^i = \{\pi \in \Pi_k: N(\pi,k) = i\}$.
To enable a consistent comparison of paths across different classes, we impose a total order such that $k<\ell \Rightarrow \Pi_k^i \prec \Pi_\ell^j$ $\text{for all } i,j$ and $i<j \Rightarrow \Pi_k^i \prec \Pi_k^j$.
This order ensures that any two paths with the same start and goal nodes can be compared.

\begin{algorithm}[tb]
\caption{Hierarchical Class-ordered A* (HCOA*)}\label{alg:HCOAstar}
\hspace*{\algorithmicindent} \textbf{Input:} {$G$, $v_s$, $v_g$, $h(\cdot)$} \\
\hspace*{\algorithmicindent} \textbf{Output:} {$\pi^0$}
\begin{algorithmic}[1]

    \ForAll{$ \ell = \ell_{n-1}, \ldots, \ell_0$}
        \State $v_s^\ell \gets p(v_s,\ell);\hspace{0.5em} v_g^\ell \gets p(v_g,\ell)$
        \State \( g(v_s^\ell) \gets 0;\hspace{0.5em} \theta(v_s^\ell) \gets 0 \cdot \mathbf{1}_K;\hspace{0.5em} f(v_s^\ell) \gets h(v_s^\ell) \)
        \State \( g(v^\ell) \gets \infty;\hspace{0.5em} \theta(v^\ell) \gets \infty \cdot \mathbf{1}_K, \quad \forall v^\ell \in G^{\ell'} \setminus \{v_s^\ell\}\)
        \State PredecessorMap \( \gets \emptyset \) \Comment{Track path reconstruction}
        \State \( Q \gets \{ v_s^\ell \} \) \Comment{Initialize priority queue}
        
        \While{Q is not empty}
            \State \( v \gets\)\textsc{PopNode}($Q$);\hspace{0.5em}$Q \gets Q \setminus \{v\}$
            \If{ \( v = v_g^\ell \) }
                \State $\pi^{\ell} \gets$ \textsc{Path}(PredecessorMap, \( v_g \))
                \State \textsc{Break}
            \EndIf
            
            \ForAll{$ u \in \text{neighbors}(v, G^{\ell'}) $}
                \State $\theta(v,u) \gets $ \textsc{Semantics}($\phi_V^\ell(v), G^{0'}(u) $)
                \If{ \big(\( \theta(v) + \theta(v,u) \prec \theta(u) \)\big) \textbf{or} \\ \hspace{4.0em} \big($\theta(v) + \theta(v,u) = \theta(u)$ \textbf{and} \\ \hspace{4.4em} $g(v) + w^\ell(v,u) < g(u)$\big)}
                    \State \( Q \gets Q \cup \{u\} \)
                    \State PredecessorMap[\( u \)] \( \gets v \)
                    \State $\theta(u) \gets \theta(v) + \theta(v,u)$
                    \State  $g(u) \gets g(v) + w^\ell(v,u)$
                    \State $f(u) \gets g(u) + h(u)$
                \EndIf
            \EndFor
        \EndWhile

        \ForAll{$ \hat\ell <\ell$}
            \State $G^{\hat\ell'} \gets \{v \in G^{\hat\ell'}: p(v,\ell) \in \pi^\ell\}$
        \EndFor

    \EndFor
    \State \Return $\pi^0$
\end{algorithmic}
\end{algorithm}

\section{Problem Formulation}

Consider a mobile robot tasked with navigating a complex 3D environment where certain regions have lower traversal priority.
The robot is provided with a HSG of the environment (e.g., 3DSG), as outlined in Section \ref{sec:preliminaries}, where the semantic classes of the nodes in layer $\ell = 0$ are assigned based on perception, and tailored to the specific task (e.g., safety constraints).
Let $v_s \in V^0$ and $v_g\in V^0$ denote the robot’s starting and goal nodes, respectively.
The objective is to determine the shortest path in $\ell =0$ while minimizing traversal through the least favorable edges. 
Formally,
\begin{subequations}\label{eq:optimal_path}
    \begin{eqnarray}
         & \pi^*(v_s,v_g) = \argmin\limits_{\pi \in \Pi^*(v_s,v_g)}{\sum\limits_{e \in \pi}w^0(e)}, \label{eq:optimal_path_a}\\
         & \Pi^* = \min\limits_{i \in \mathbb{N}} \min\limits_{k \in \mathcal{K}} \Pi_k^i, \label{eq:optimal_path_b}
    \end{eqnarray}
\end{subequations}
where $\Pi^*$ is the set of paths obtained by minimizing over all possible classes $k$ and number of least favorable class $i$.

However, due to computational constraints, the robot seeks to avoid a full graph search over the entire $G^0$, as its structure is both large and semantically diverse, potentially rendering the search intractable.

\subsection{Problem Statement}
We propose a hierarchical semantic path-planning algorithm for efficient robot navigation in large-scale environments. 
The algorithm operates top-down across the layers of the HSG, iteratively computing the optimal semantic path at each layer while pruning nodes that are not included in the path.
Additionally, we introduce three methods for node classification to predict the semantic classes of higher-layer nodes: a Majority-Class, a kNN, and a GNN method.

\section{Path-Planning}

\subsection{Hierarchical Class-ordered A*}
\label{HCOA*}

We introduce Hierarchical Class-ordered A* (HCOA*), a hierarchical semantic path-planning algorithm for HSGs.
The algorithm initiates planning at layer $\ell = n - 1$, and recursively proceeds to lower layers. 
At each layer, it prunes the HSG based on the computed path from the previous layer and applies the same planning algorithm to the pruned HSG. 
Within each layer, the algorithm utilizes Class-ordered A* (COA*) \cite{Lim2020}, which finds the shortest path while minimizing the number of least favorable edges through lexicographic comparison.
By leveraging the hierarchy, the algorithm performs multiple searches on smaller subgraphs rather than a potentially computationally expensive search over $G^0$.

HCOA* is presented in Algorithm \ref{alg:HCOAstar}.
We use the function $h: V \to \Re$ to denote an admissible heuristic function, similar to standard A*.
Lines 2–6 initialize the variables, where the function $p(v,\ell)$ returns the ancestor of node $v$ in layer $\ell$.
Lines 7–24 execute COA*, which will be detailed in Section \ref{COA*}.
Notably, Line 14 computes the semantic class of the edge $(v,u)$ based on the semantic classes of nodes $v$ and $u$.
If $\ell = 0$, the nodes' semantics are determined from perception data and the task, whereas for $\ell \neq 0$, they are predicted using the methods described in Section \ref{sec:classification}.
Additionally, $G^{0'}(u)$ is the induced subgraph of node $u \in V^\ell$ on layer $0$ given by $G^{0'}(u) = \{u' \in  V^0: p(u', \ell) = u, \: \text{where} \: u \in V^\ell  \}$, 
while $G^{\ell'}$ is the pruned graph at layer $\ell$ (Line 27). 
Lines 25–27 refine the HSG by pruning nodes that do not share an ancestor with the paths in the higher layers.

\subsection{Class-ordered A*}
\label{COA*}

In this section, we present an overview of the simplified COA* implementation used in HCOA* to help the reader understand the simplifications made to the baseline algorithm.
Further details, along with proofs of the COA* optimality and completeness, 
can be found in~\cite{Lim2020}.

Consider a single-layer, weighted, semantic graph $G^\ell$. 
We characterize each node $v$ by the triple $(q, g, \theta_V)$,
where $q$ is the predecessor node (saved in the PredecessorMap in Algorithm~\ref{alg:HCOAstar}), $g: V^\ell \to \Re^+$ is the cost-to-come function, and $\theta_V: V^\ell \to \mathbb{N}^K$ specifies the number of edges of each semantic class along the path up to node $v$.
We extend $\theta_V$ to single edges by defining $\theta_E: E^\ell \to \mathbb{N}^K$ as an one-hot vector with all elements zero except at $\phi_E$. 
The subscript is omitted in Algorithm \ref{alg:HCOAstar}.
The function $\textsc{PopNode}$ selects the next node to expand, and $\textsc{Path}$ constructs the optimal path by backtracking through the predecessor nodes from the goal. 

\begin{figure}[tb]
     \centering
     \begin{subfigure}[b]{0.49\linewidth}
         \centering
         \includegraphics[width=\textwidth]{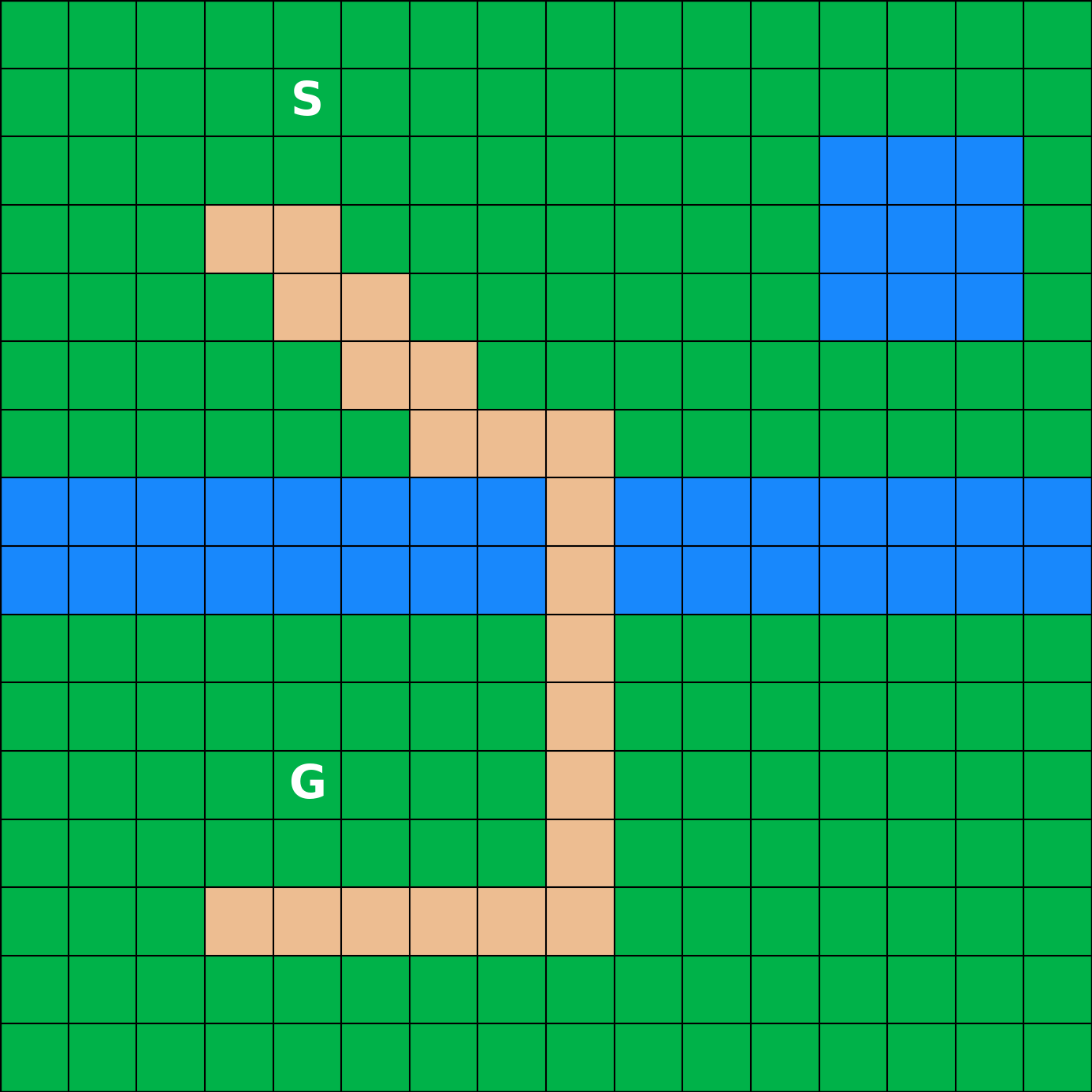}
         \caption{}
         \label{fig:ex_environ}
     \end{subfigure}
     \begin{subfigure}[b]{0.49\linewidth}
         \centering
         \includegraphics[width=\textwidth]{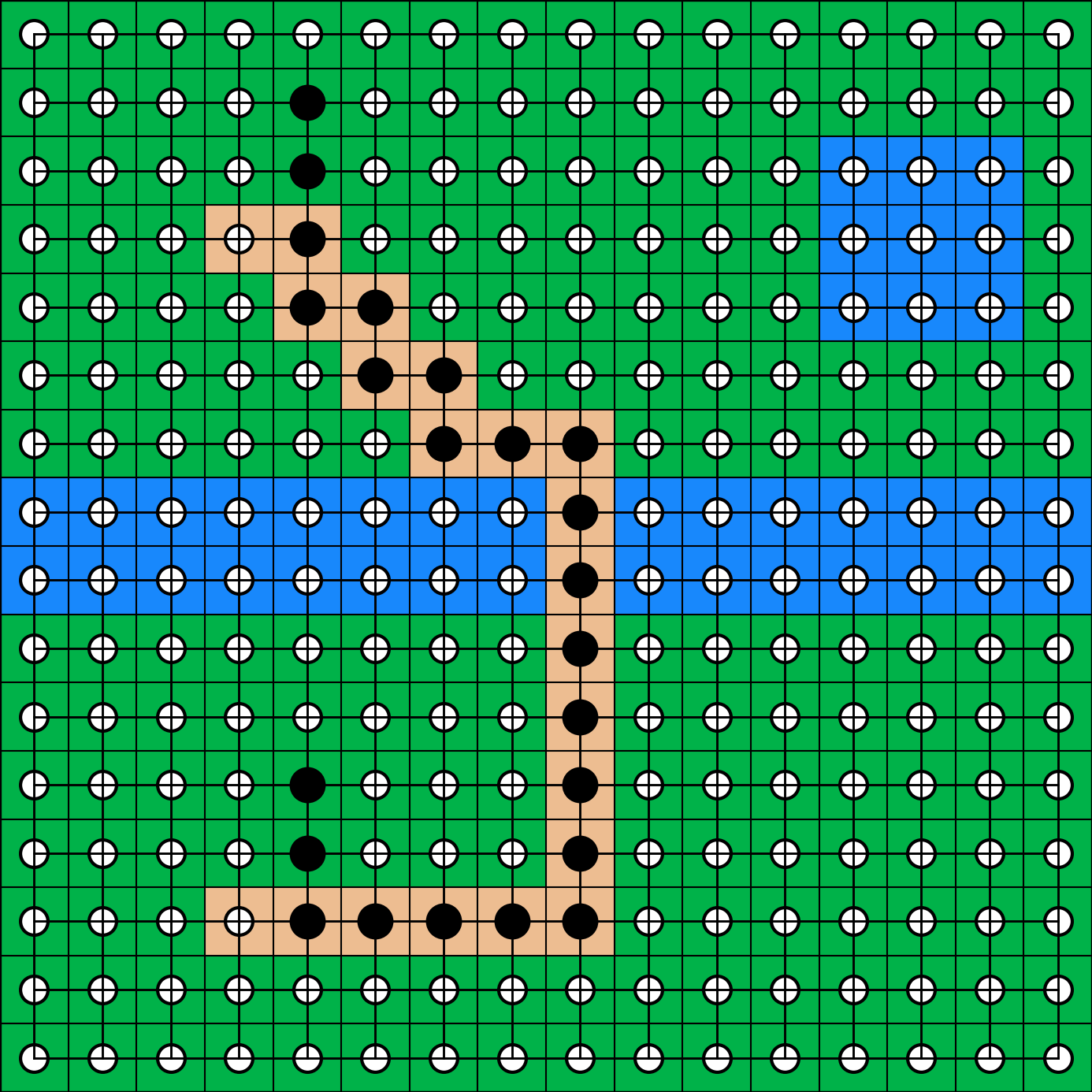}
         \caption{}
         \label{fig:ex_graph}
     \end{subfigure}
        \caption{(a) Grid world with three semantic classes: road, grass, river. \(\textbf{S}\) is the starting cell while \(\textbf{G}\) is the target; (b) Graph of the environment and path produced by COA*.} 
        \label{fig:example}
\end{figure}

Figure~\ref{fig:example} presents an example of COA* in a grid world, where the robot must navigate from $\mathbf{S}$ to $\mathbf{G}$ while satisfying safety constraints prioritized in descending order: (i) avoid water, and (ii) avoid grass. 
To enforce these constraints, we assign semantic labels to nodes via  $\phi_V^\ell$.

\subsection{Performance Guarantees}

The following propositions establish the algorithm's completeness along with sufficient conditions for  optimality.
For both propositions, the assumptions for completeness and optimality of COA* hold \cite{Lim2020}.
We conclude the section with a discussion on the algorithm’s computational complexity.

\begin{proposition}
    Let $G$ be an HSG and let nodes $v_s, v_g \in V^0$.  
    HCOA* is \textit{complete}, meaning that it is guaranteed to find a path $\pi^0 = \pi(v_s,v_g)$, whenever one exists.
\end{proposition}

\begin{proof}
Let $\pi^0$ exist. 
The structure of $G$ ensures that for each layer $\ell \in L$, there exists a corresponding path $\pi^\ell = \pi(v_s^\ell, v_g^\ell)$, where $v_s^\ell = p(v_s, \ell)$ and $v_g^\ell = p(v_g, \ell)$.  
At each subgraph $G^\ell$, COA* is guaranteed to find the optimal path $\pi^{\ell*}$, if one exists \cite{Lim2020}.  
Suppose, ad absurdum, that the graph pruning performed in Lines 25–27 results in $\Pi(v_s^{\ell-1}, v_g^{\ell-1}) = \emptyset$, preventing COA* from finding a solution in layer $\ell-1$. 
This implies that $\pi^{\ell*}$ is disconnected, contradicting COA*'s completeness.  
\end{proof}

\begin{proposition}\label{prop:optimal}
Let $G$ be an HSG with nodes $v_s, v_g \in V^0$.  
Suppose that in layer $\ell = 1$, the acyclic path $\pi^1(v_s^1, v_g^1)$ is unique, and that for all $u_s, u_g \in V^0$ sharing the same parent, $P = p(u_s, 1) = p(u_g, 1)$, every node $u$ along the optimal path $\pi^{*}(u_s, u_g)$ satisfies $p(u,1) = P$.  
Then, HCOA* is guaranteed to find the optimal path $\pi^*(v_s, v_g)$.
\end{proposition}

\begin{proof}
Define $G^{0'} \subseteq G^0$ as the pruned graph after $n$ iterations of HCOA*.
Suppose, ad absurdum, that a segment of the optimal path $\pi^*(v_s, v_g)$ has been removed, i.e., there exists a sub-path $\widehat{\pi} \subseteq \pi^*(v_s, v_g)$ such that $\widehat{\pi} \not\subset G^{0'}$.
Two cases arise: 
(i) The parents of $\widehat{\pi}$ in layer $\ell = 1$ form a cycle that starts and ends at a parent of a node in $\pi^*(v_s, v_g) \cap G^{0'}$. 
This contradicts the second assumption, which states that the nodes of the optimal path at $\ell=0$, where the starting and goal node share the same parent, must all have the same parent. 
(ii) The parents of $\widehat{\pi}$ in layer $\ell = 1$ are part of a sub-path that starts at the parent of one node in $\pi^*(v_s, v_g) \cap G^{0'}$ and ends at a different parent of another node in $\pi^*(v_s, v_g) \cap G^{0'}$.
This contradicts the first assumption, as it implies the existence of two distinct paths in layer $\ell = 1$, violating uniqueness.
\end{proof}

The two assumptions of Proposition~\ref{prop:optimal} may seem restrictive at first, but they often hold in indoor environments. 
Figure~\ref{fig:prop2} shows an example, where $\mathbf{G}$ denotes the goal, $R(\cdot)$ denotes a room, and $C_1$, $C_2$ represent specific conditions, namely $C_1$: open door and $C_2$: dangerous area. 
When both $C_1$ and $C_2$ are false, Proposition~\ref{prop:optimal} holds, and HCOA* successfully finds the optimal path (shown in green).
However, when $C_1$ is true, the path produced by HCOA* (still shown in green) becomes suboptimal. 
This occurs because the room is represented, in this example, by its center node.
For large rooms such as corridors (e.g., $R(1)$), this representation becomes misleading.
Additionally, if $C_2$ is true, the optimal path requires moving from $R(2)$ to $R(1)$ and then back to $R(2)$. 
This cycle cannot be captured by COA* operating at the room layer, as it computes acyclic paths. 
Hence, HCOA* produces a suboptimal path in the places layer in this case (shown in blue).

In general, the spatial representation of nodes in higher layers, as well as their computed semantic classes, can significantly affect HCOA*'s performance, potentially leading to suboptimal solutions. 
Nevertheless, as demonstrated in Section~\ref{sec:simulations}, HCOA* manages to find the optimal path in most tested scenarios.
One potential improvement is to refine the representation of rooms by associating multiple nodes with each room, for example, using nodes near doors rather than using a single node at the geometric center of the room.

\begin{figure}[tb]
    \centering        
    {\includegraphics[width=1\linewidth]{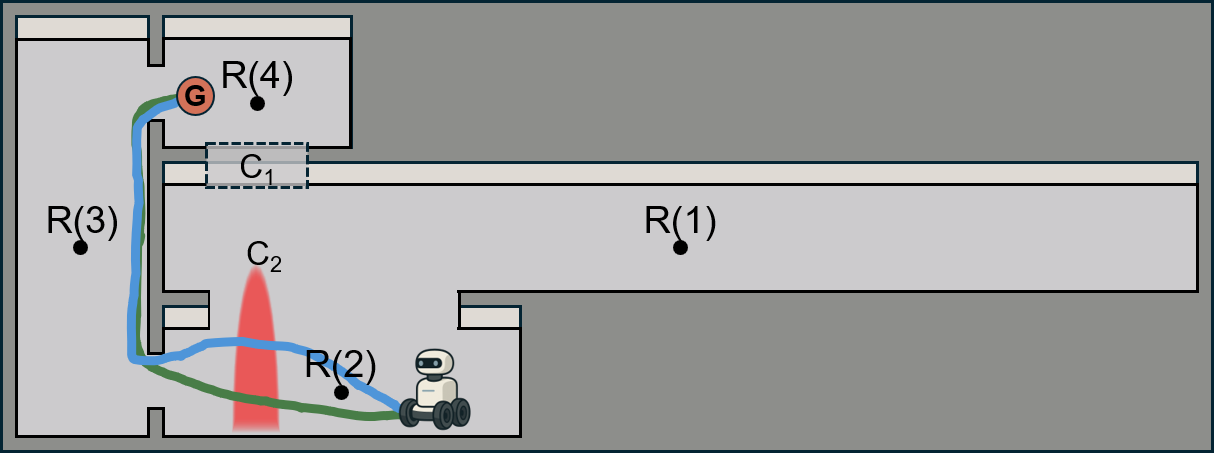}}
    \caption{Example demonstrating Proposition~\ref{prop:optimal}. 
    $\mathbf{G}$ denotes the goal location, $R(\cdot)$ indicates rooms, and $C_1$, $C_2$ represent specific conditions: $C_1$: open door, and $C_2$: dangerous area.} 
    \label{fig:prop2}
\end{figure}

\begin{figure*}[!t]
    \centering   
    {\includegraphics[width=1\linewidth]{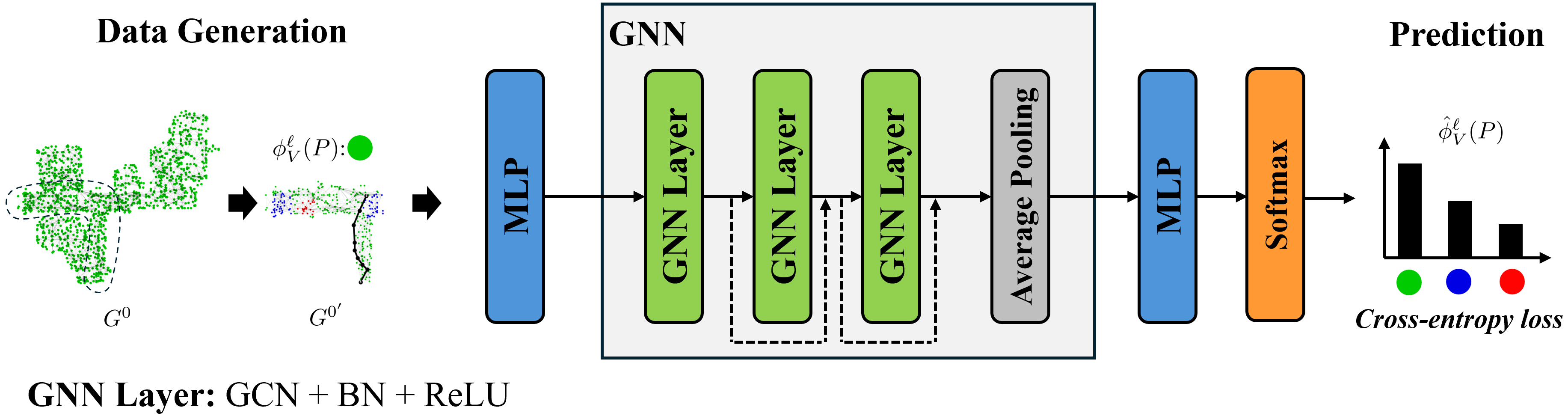}}

    \caption{Proposed GNN architecture for node $P \in V^\ell$ ($\ell \neq 0$) classification, utilizing the semantic classes (green, blue, red) of nodes in $\ell=0$. The dataset $\mathcal{D} = \{G^{0'}(P),\phi_V^\ell(P)\}$ is generated by running COA* on induced subgraphs of nodes. The network consists of two MLPs for pre-processing and post-processing, three GNN layers with skip connections, and an average pooling operator. Training is conducted using cross-entropy loss over the semantic classes.}
    \label{fig:GNN}
\end{figure*}

Regarding the computational complexity of the algorithm, note that, in the worst case, COA* with an admissible heuristic has the same complexity as Dijkstra’s algorithm, which is $\mathcal{O}(|E^0| + |V^0| \log |V^0|)$ when implemented with a Fibonacci heap. 
Hence, the worst-case complexity of HCOA* is given by $\mathcal{O}\left(\sum_{\ell=0}^{L} \left( |E^{\ell'}| + |V^{\ell'}| \log |V^{\ell'}| \right)\right)$, where $|V^\ell{}'|$ and $|E^\ell{}'|$ denote the number of vertices and edges, respectively, of the pruned graph at each layer $\ell$.

\section{Semantic Class Prediction}
\label{sec:classification}

Let $\phi_V^\ell: V^\ell \to \mathcal{K}$ be the function that assigns semantic classes to nodes in layers $\ell =1,\ldots,n$.
Ideally, given an optimal path $\pi^*(v_s,v_g)$ computed in $G^0$, the semantic class of a node $P \in V^\ell$ ($\ell = 1,\ldots,n$) is determined conservatively by selecting the highest semantic class present in the nodes in $\ell =0$ that both belong to the optimal path and are part of the subgraph $G^{0'}(P) = \{u \in  V^0: p(u, \ell) = P, \: \text{where} \: P \in V^\ell \}$ induced by $P$ on layer $0$. 
Formally,
\begin{subequations}\label{eq:phi_prime}
    \begin{align}
         & \phi_V^\ell(P) = \max\limits_{u \in  \pi'(P)}{\phi_V^{0}(u)}, \label{eq:phi_prime_a}\\
         & \pi'(P) = \pi^*(v_s,v_g) \cap G^{0'}(P). \label{eq:phi_prime_b}
    \end{align}
\end{subequations}
This process follows a bottom-up approach, necessitating the computation of the optimal path in layer $\ell=0$.
Given that HCOA* works top-down, we seek alternative methods to find the semantic classes of the higher layers.
We introduce three methods for predicting the semantic class of higher-layer nodes: a Majority-Class (MC), a k-Nearest Neighbors (kNN), and a Graph Neural Network (GNN) approach. 
The latter two methods require a supervised training phase, which necessitates the construction of a dataset.

To this end, we construct a dataset $\mathcal{D} = \{G^{0'}(P),\phi_V^\ell(P)\} $ by selecting nodes $P \in V^\ell$ for $\ell \neq 0$ and extracting their induced subgraphs $G^{0'}(P)$. 
We then assign semantic classes to the nodes in $G^{0'}(P)$ based on the given task. 
For navigation tasks involving object avoidance, we determine region priorities by placing disks of a specified radius and semantic class around certain objects. 
Next, we run COA* on $G^{0'}(P)$ and compute $\phi_V^\ell(P)$, where the start is a border node.

A border node $v^* \in V^0$ is a node that is connected by an edge to another (border) node $u^* \in V^0$ whose ancestor in layer $\ell$ is different from that of $v^*$.
Formally, there exists $e = (v^*, u^*)$ where $v^*, u^* \in V^0$ and $p(v^*, \ell) \neq p(u^*, \ell)$.
Border nodes play a crucial role in the classification of $P \in V^\ell$. 
If the start and goal nodes do not share the same ancestor, the optimal path must traverse a border node in layer $\ell$. 
Conversely, if both nodes have $P$ as their ancestor, the classification of $P$ becomes irrelevant, as HCOA* will not need to compute a path in layer $\ell$.

\subsection{Majority-Class Method}
\label{MC}

The Majority-Class (MC) method is a simple and fast approach for predicting the semantic class of a node $P \in V^\ell$ for $\ell \neq 0$. 
The predicted class is determined by counting the occurrences of each semantic class in the induced subgraph $G^{0'}(P)$ and selecting the most frequent one.
That is,
\begin{equation}\label{eq:MC_approach}
    \widehat{\phi}_V^\ell(P) = \argmax\limits_{k \in \mathcal{K}}{N_V(G^{0'}(P),k)},
\end{equation}
where $N_V(G,k) = |\{v \in G: \phi_V^0(v) = k\}|$ is the number of nodes of class $k$ in $G$.

\subsection{k-Nearest Neighbors Method}
\label{kNN}

The kNN algorithm \cite{Cover1967_kNN} is a non-parametric, instance-based learning method that classifies data points based on the majority vote of their $\mathsf{k}$ nearest neighbors in the feature space.
Given the induced subgraph $G^{0'}(P)$ with node set $V^{0'}$, we extract a graph-level feature vector $\mathsf{f}$ consisting of: (i) the proportion of each semantic class in the graph
\begin{equation}
\frac{N_V(G^{0'}(P),k)}{|V^{0'}|}, \quad k \in \mathcal{K},
\end{equation}
and (ii) the majority class among the border nodes $\argmax_{k \in \mathcal{K}}{N_{V^*}(G^{0'}(P),k)}$, where $V^*$ denotes the set of border nodes and $k$ the index of a semantic class (not to be confused with $\mathsf{k}$ used in kNN).
We intentionally limit the number of features, as the inference time of kNN increases with the dimensionality of the feature space.

To ensure uniform scaling, we standardize all feature vectors $\mathsf{f}$. 
Graph similarity is measured via Euclidean distance, and classification is performed by assigning the majority class among the $\mathsf{k}$ nearest neighbors from the training set:
\begin{equation}\label{eq:kNN_approach}
    \widehat{\phi}_V^\ell(P) = \argmax\limits_{k \in \mathcal{K}} \sum\limits_{i=1}^{\mathsf{k}} \mathbbm{1}[\phi_V^\ell(P_i) = k],
\end{equation}
where $P_1, P_2, \ldots, P_{\mathsf{k}}$ are the $\mathsf{k}$ nearest neighbors of $P$ in the training set, and $\mathbbm{1}[\cdot]$ is the indicator function that equals 1 if the condition is true and 0 otherwise.

\subsection{Graph Neural Network Method} 
\label{GNN}

The GNN method is a supervised learning approach that leverages the graph structure and features through message passing.
The proposed model is shown in Figure~\ref{fig:GNN}. 
We use border nodes as input features, which are concatenated with the semantic classes of the nodes in $G^{0'}$ represented as one-hot encodings.
A 2-layer MLP is used to preprocess the input.

The GNN consists of three layers, each incorporating a Graph Convolutional Network (GCN) operator \cite{kipf2017semi} for message passing, followed by batch normalization and a ReLU activation. 
We adopt a 3-layer architecture, as deeper GNNs tend to suffer from oversmoothing and slower inference, both of which are critical limitations in path-planning applications.
To mitigate oversmoothing, we introduce skip connections between the GNN layers. 
Additionally, we apply average pooling to handle environments of varying sizes.
The pooled representation is then processed through another 2-layer MLP, followed by a softmax function. 
Finally, we use cross-entropy as the loss function for node classification. 
Formally, given a dataset $\mathcal{D}$ and learnable parameters $\vartheta_{\text{GNN}}, \vartheta_{\text{MLP}}$, the predicted semantic class is obtained by
\begin{equation}\label{eq:ce_loss}
\min\limits_{\vartheta_{\text{GNN}}, \vartheta_{\text{MLP}}} {\sum\limits_{(G^{0'}(P),\phi_V^\ell(P))\in \mathcal{D}} {\mathcal{L}_{\text{CE}} \big(\widehat{\phi}_V^\ell(P), \phi_V^\ell(P) \big)} },
\end{equation}
where $\widehat{\phi}_V^\ell(P) = f(G^{0'}(P); \vartheta_{\text{GNN}}, \vartheta_{\text{MLP}})$ represents the model function parametrized by $\vartheta_{\text{GNN}}, \vartheta_{\text{MLP}}$.

\section{Simulations}\label{sec:simulations}

We perform simulations on two publicly available datasets generated with Hydra \cite{hughes2022hydra}: the uHumans2 office 3DSG (Figure~\ref{fig:3dsg}) and the subway 3DSG (Figure~\ref{fig:subway}(\subref{fig:subway_3dsg})).
Figure~\ref{fig:3dsg} also shows the layer names.
The \textsf{place} layer ($\ell=0$) is a graph  with nodes as obstacle-free locations and edges indicating traversability.
The \textsf{room} layer ($\ell=1$) consists of nodes representing room centers and edges connecting neighboring rooms.

Our task involves navigation with object or room avoidance.  
The pipeline proceeds as follows: we first translate safety-related commands (e.g., ``avoid computers") into semantic classes at the \textsf{place} layer, and impose a total order based on their priority.  
Next, depending on the room classification method employed, we train the corresponding models to predict semantic room classes.  
Finally, we select the best-performing model and execute HCOA* to compute the path.

In all simulations, we impose two safety constraints resulting in three semantic classes $\mathcal{K}$, ordered by decreasing priority.
We leverage the \textsf{room} and \textsf{object} layers to identify relevant rooms and objects, and assign semantic labels to nearby \textsf{place} nodes accordingly.
The resulting classes are visualized (in decreasing priority) as: 1 (Green), 2 (Blue), and 3 (Red).

The first section presents the comparison of the three methods for semantic class prediction.
The second section showcases path-planning scenarios in the office 3DSG, comparing HCOA* to COA*. 
We use the same weight function across all graph layers and denote it by $w$ for simplicity, corresponding to the Euclidean distance between nodes.
We also include results from a modified A* algorithm, denoted MA* (similar to~\cite{serdel2023_smana}), to demonstrate the importance of imposing a total order over semantic classes. 
In this variant, the edge cost in A* is modified to $w' = w + \alpha^k$, where $k \in \mathcal{K}$ and $\alpha$ is a scaling factor.  
The third section presents analogous experiments in the subway 3DSG, also comparing HCOA* to COA* and MA*.
In these experiments, HCOA*-GNN utilizes the best GNN model from the previous section for room inference.

All the simulations were performed using Python 3.8.10 and PyTorch 2.4.1 on a computer with 2.2 GHz, 12-Core, Intel Core i7-8750H CPU, 16GB RAM and an Nvidia RTX 2060 GPU, 6GB VRAM.

\subsection{Performance Metrics}

To evaluate the performance of the prediction of a room's semantic class, we compute the accuracy, which measures the proportion of correctly classified rooms in each dataset:
\begin{equation}\label{eq:accuracy}
\text{Accuracy} = \frac{\sum\limits_{i = 1}^n \mathbbm{1}[\widehat{\phi}_V^\ell(P_i) = \phi_V^\ell(P_i)]}{n},
\end{equation}
where $n$ is the number of samples.
For the planning scenarios, we compute the algorithms' computational time along with the number of expanded nodes (Line 8 in Algorithm \ref{alg:HCOAstar}). 

\subsection{Semantic Class Prediction}

\begin{table}[!t]
\caption{Semantic Class Prediction.}
\label{tab: results1}
\centering
\small
\begin{tabular}{ |c|c|c|c| } 
    \hline
    {Metrics} &  {MC} &  {kNN} & {GNN}   \\
    \hline
    Training Time & \textbf{-} & 10 s & 97 min  \\
    \hline
    Validation Acc. (\%) & 42.79  & 52.43  &  \textbf{63.43}  \\
    \hline
    Test Acc. (1-20 bn Rooms) (\%) & 57.00 & 70.50 & \textbf{74.00}   \\ 
    \hline
    Test Acc. (21-30 bn Rooms) (\%) & 33.75 & 48.50 & \textbf{58.25}  \\ 
    \hline
    Test Acc. (31-40 bn Rooms) (\%) & 37.00 & 44.25 &  \textbf{54.75} \\ 
    \hline
    Test Acc. (41-50 bn Rooms) (\%) & 41.50 & 43.00 &  \textbf{47.50} \\ 
    \hline
\end{tabular}
\end{table}

In this section, we present the results from predicting the semantic class of \textsf{room} nodes.
To generate the dataset, we constructed 14,000 graphs by extracting the induced subgraphs from all rooms in the Office 3DSG and randomly assigning semantic classes to nodes within a disk of randomly chosen center locations, repeating this process 2,000 times per room.
Then, we ran COA* to determine the semantic class of the room, starting from a randomly selected border node.
The dataset was split into 80\% training, 10\% validation, and 10\% testing. The GNN model was trained using the Adam optimizer \cite{kingma2015adam} with a learning rate of $10^{-2}$, over 1,600 epochs, a dropout rate of 0.2 and a batch size of 64. The GNN and MLP layers contain 32 neurons.
For training the kNN model, we set $\mathsf{k} = 5$, and used half of the training dataset.

Table~\ref{tab: results1} summarizes the results of \textsf{room} node classification. 
For the test set, we analyzed the performance separately for rooms with different numbers of border nodes (bn). 
The results indicate that the GNN approach outperforms the other two methods. 
This result is expected, as the GNN learns complex spatial, structural, and semantic features, unlike the MC baseline and the kNN method, which relies on hand-crafted features.
Furthermore, we observe that all methods achieve higher accuracy for rooms with fewer bns. 
These rooms are typically smaller and have only one entrance, resulting in spatially clustered border nodes. 
In such cases, features considering the border nodes' semantics lead to improved classification performance.
The training accuracy of the GNN is 67.94\%, while that of the kNN is 60.50\%.
Note that the training and validation accuracies reported in Table~\ref{tab: results1} correspond to the full set of rooms, irrespective of the number of border nodes.

\subsection{Path-Planning on uHumans2 Office Scene}
\label{sec:path_planning_office}

\begin{figure}[tb]
     \centering
     \begin{subfigure}[b]{0.49\linewidth}
         \centering
         \includegraphics[width=\textwidth]{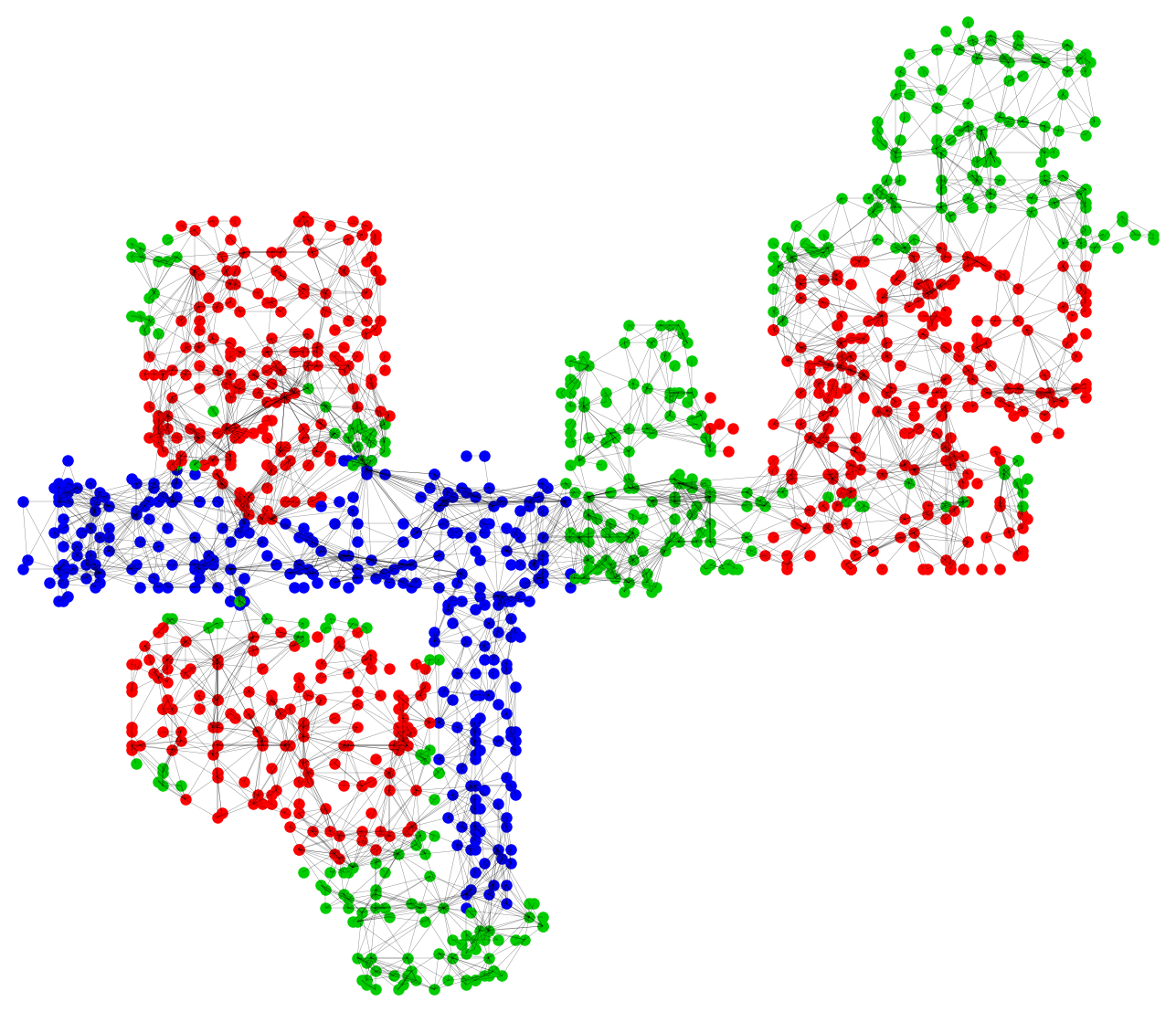}
         \caption{}
         \label{fig:office_places}
     \end{subfigure}
     \begin{subfigure}[b]{0.49\linewidth}
         \centering
         \includegraphics[width=\textwidth]{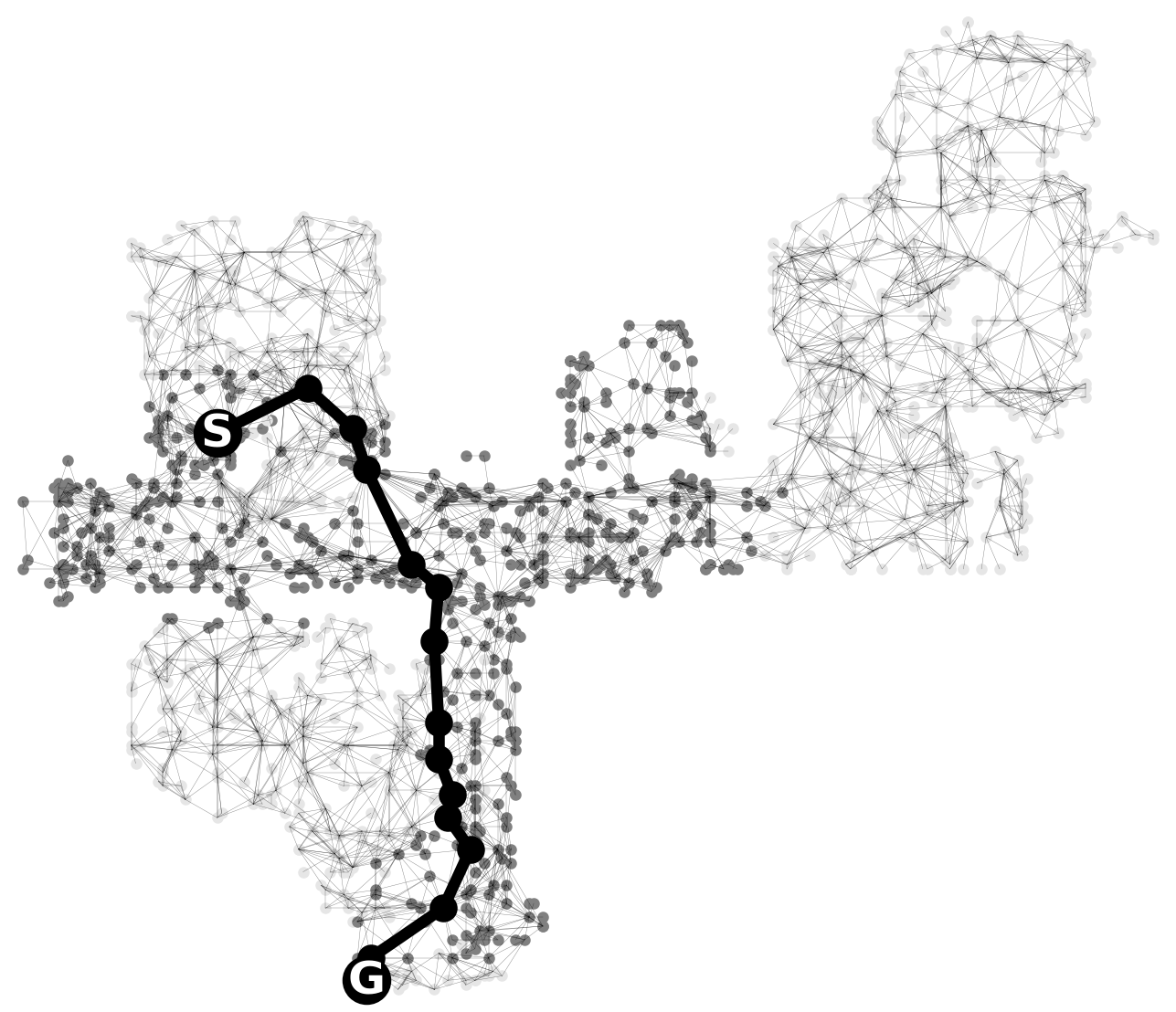}
         \caption{}
         \label{fig:office_COAStar_places}
     \end{subfigure}
     
     \begin{subfigure}[b]{0.49\linewidth}
         \centering
         \includegraphics[width=\textwidth]{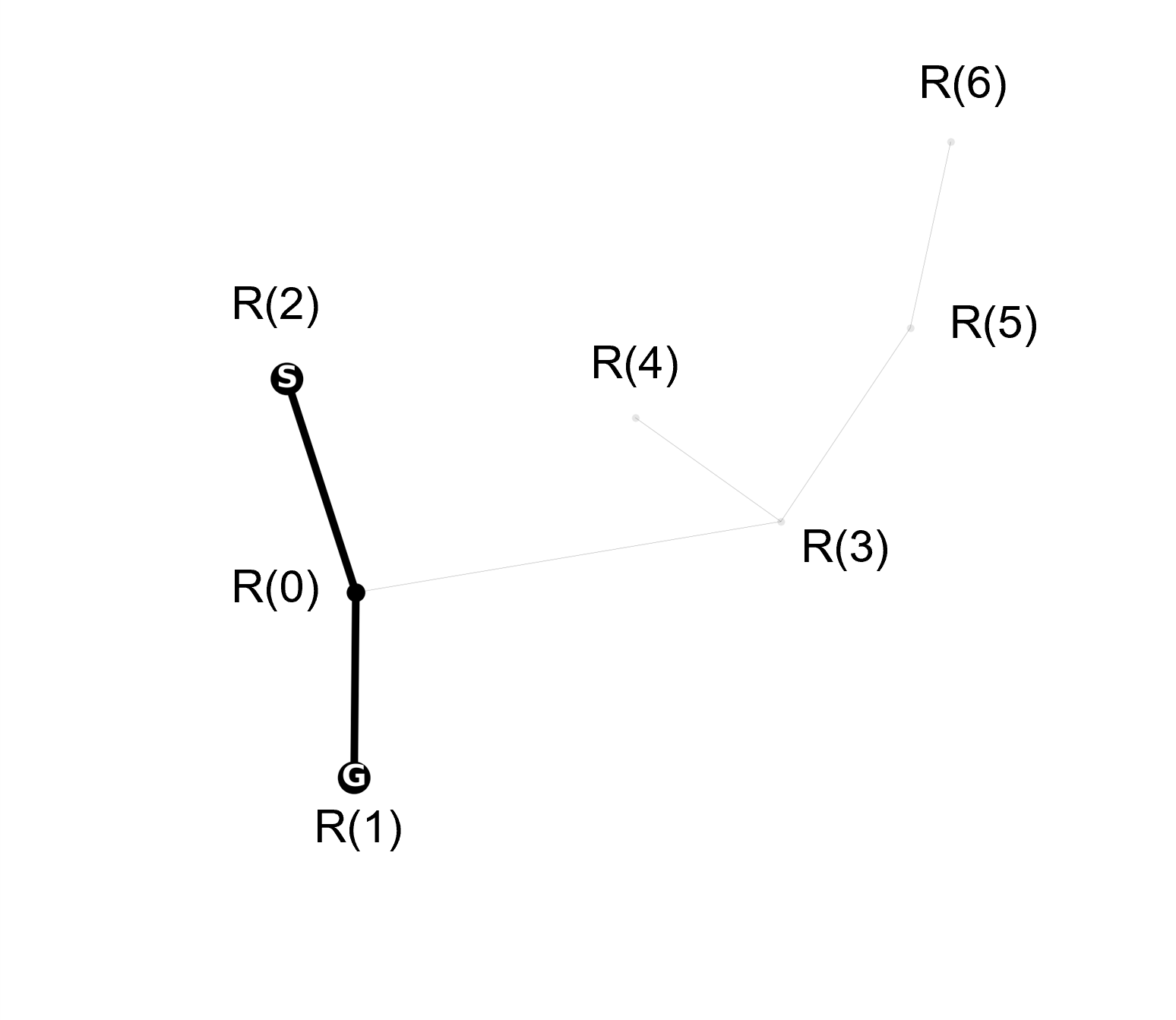}
         \caption{}
         \label{fig:office_HCOAStar_rooms}
     \end{subfigure}
     \begin{subfigure}[b]{0.49\linewidth}
         \centering
         \includegraphics[width=\textwidth]{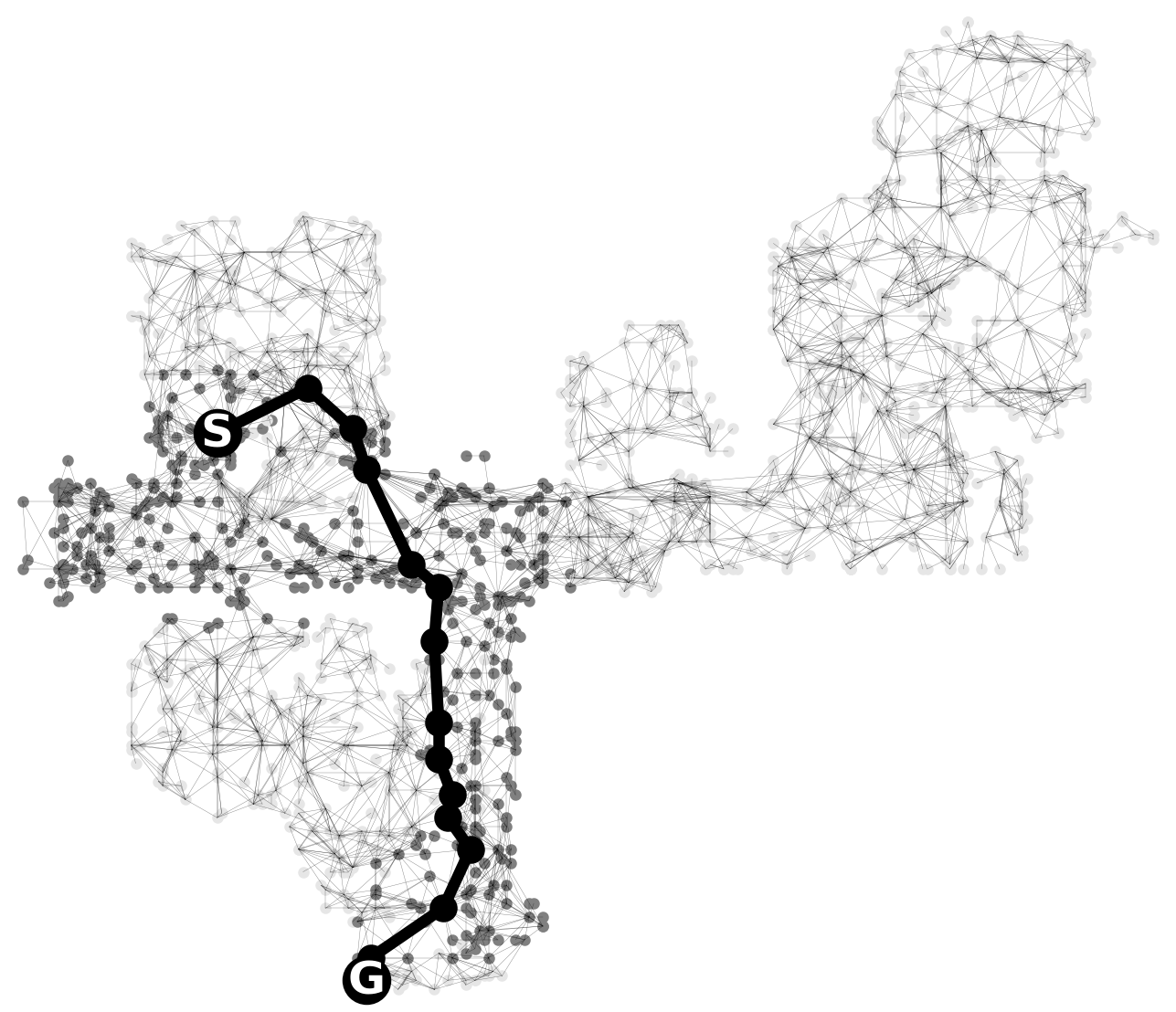}
         \caption{}
         \label{fig:office_HCOAStar_places}
     \end{subfigure}

     \caption{(a) \textsf{Place} subgraph of the uHumans2 office (52m×45m) with its semantic classes (green, blue, red); (b-d) The starting node is denoted by \textbf{S}, and the goal node by \textbf{G}. The path of each algorithm is shown in black. Expanded nodes are depicted in dark gray, while unexpanded nodes in light gray: (b) Path of COA* on the \textsf{place} subgraph; (c) Path of HCOA* on the \textsf{room} subgraph; (d) Path of HCOA* on the \textsf{place} subgraph.} 
     \label{fig:office}
\end{figure}

In this section, we perform multiple path-planning scenarios in the office 3DSG shown in Figure~\ref{fig:3dsg}. 
The robot must navigate from a start location to a goal while satisfying the following safety constraints, prioritized in descending order: (i) avoid computers, and (ii) avoid entering room R(0), a typically crowded area.
To enforce these constraints, we assign semantic labels to \textsf{place} nodes. 
Specifically, we set $\phi_V^0(v) = 3$ for all $v \in C$, where $C$ denotes the set of \textsf{place} nodes within a radius $r = 3$m of any computer in the scene. 
Formally, $C = \{ v \in G^0: \| x(v) - x(o)\|_2 \leq r ,\hspace{0.5em}  \forall o \in O\}$
where $x(\cdot)$ denotes the location of a \textsf{place} node $v$ or \textsf{object} $o$, and $O$ is the set of all computers in the 3DSG.
We also assign $\phi_V^0(v) = 2$ for all $v \in \text{R(0)}$, thereby penalizing passing through this room.
The rest of the nodes have $\phi_V^0(v) = 1$.

First, we evaluate the computational efficiency of the algorithm by performing 1,000 runs with fixed start and goal. 
The starting node $v_s$ is located in room R(2), and the goal $v_g$ is in room R(1). 
Figure~\ref{fig:office}(\subref{fig:office_places}) shows the \textsf{place} layer, where nodes are color-coded according to their semantic class.
The computed path of COA* in the \textsf{place} layer is shown in Figure~\ref{fig:office}(\subref{fig:office_COAStar_places}). 
Figures~\ref{fig:office}(\subref{fig:office_HCOAStar_rooms})-(\subref{fig:office_HCOAStar_places}) illustrate the results of HCOA* in the \textsf{room} and \textsf{place} layers. 
The figures also show the expanded nodes for each algorithm. 
In particular, COA* expanded nodes in two additional rooms compared to HCOA*. 

\begin{table}[!t]
\caption{Path-Planning on uHumans2 Office Scene.}
\label{tab: results2}
\centering
\small
\begin{tabular}{ |c|c|c|c| }
    \hline
    Algorithm & Exp. Nodes & Time ($10^{-3}$s) & Opt. Path \\
    \hline
    HCOA*-MC & \textbf{412} & \textbf{4.2} $\boldsymbol{\pm}$ \textbf{2.5} & \checkmark \\
    \hline
    HCOA*-kNN & \textbf{412} & 5.6 $\pm$ 0.2 & \checkmark \\
    \hline
    HCOA*-GNN & \textbf{412} & 11.5 $\pm$ 2.5 & \checkmark \\
    \hline
    COA* & 549 & 4.9 $\pm$ 0.4 & \checkmark \\
    \hline
\end{tabular}
\end{table}

The exact number of expanded nodes is provided in Table~\ref{tab: results2}, where HCOA* demonstrates a 25\% reduction compared to COA*.
Table~\ref{tab: results2} also presents the computational time of the algorithms. 
We observe that HCOA*-MC achieves the best performance in terms of computational efficiency, reducing the execution time by 14\% compared to COA*. 
However, HCOA*-GNN has the highest computational time, as GNN inference can be more time-consuming than graph search in small graphs (the entire \textsf{place} subgraph contains only 1,314 nodes). 
Additionally, all classification methods classify R(0) as class 2 and R(2) as class 3.
However, while GNN and kNN assign class 3 to R(3), MC classifies it as class 1. 
Although this misclassification does not impact the final path in the current scenario, it could lead to suboptimal solutions in more complex room graphs, as the robustness of the algorithm is highly dependent on the structure of the room graph. 

Finally, all four approaches successfully compute the optimal path in the \textsf{place} layer, as shown in Figure~\ref{fig:office}. 
It is important to note that while the first assumption of Proposition~\ref{prop:optimal} holds, the second assumption does not. 
However, this does not affect the results in this case, as the optimal path in the \textsf{place} layer does not require deviations into rooms outside the unique path in the \textsf{room} layer to achieve a lower-cost trajectory.

To further evaluate the suboptimality of HCOA*, we conducted 500 scenarios with varied start and goal locations. 
HCOA* successfully found the optimal path in 96.95\% of the cases. 
This outcome is highly dependent on the representation of the \textsf{room} layer; modeling a room solely by its centroid can lead to suboptimal paths, particularly in elongated or non-orthogonal rooms (e.g., R(0)) where the centroid may not accurately reflect spatial connectivity (see
Figure 3).

We also report results from the MA* implementation. 
When $\alpha = 2$, performance drops to 80.38\% (comparing to COA*), while $\alpha = 10$ yields 99.84\%. For larger values, accuracy approaches 100\%, but for $\alpha > 10^5$, it degrades (e.g., 95.25\%), indicating numerical instability. 
These findings show that MA* requires careful tuning of $\alpha$ based on the graph size, Euclidean distances, and the number of semantic classes, highlighting the benefit of using a total order as in COA* and HCOA*.

\subsection{Path-Planning on uHumans2 Subway Scene}

\begin{figure}[tb]
     \centering
     \begin{subfigure}[b]{0.99\linewidth}
         \centering
         \includegraphics[width=\textwidth]{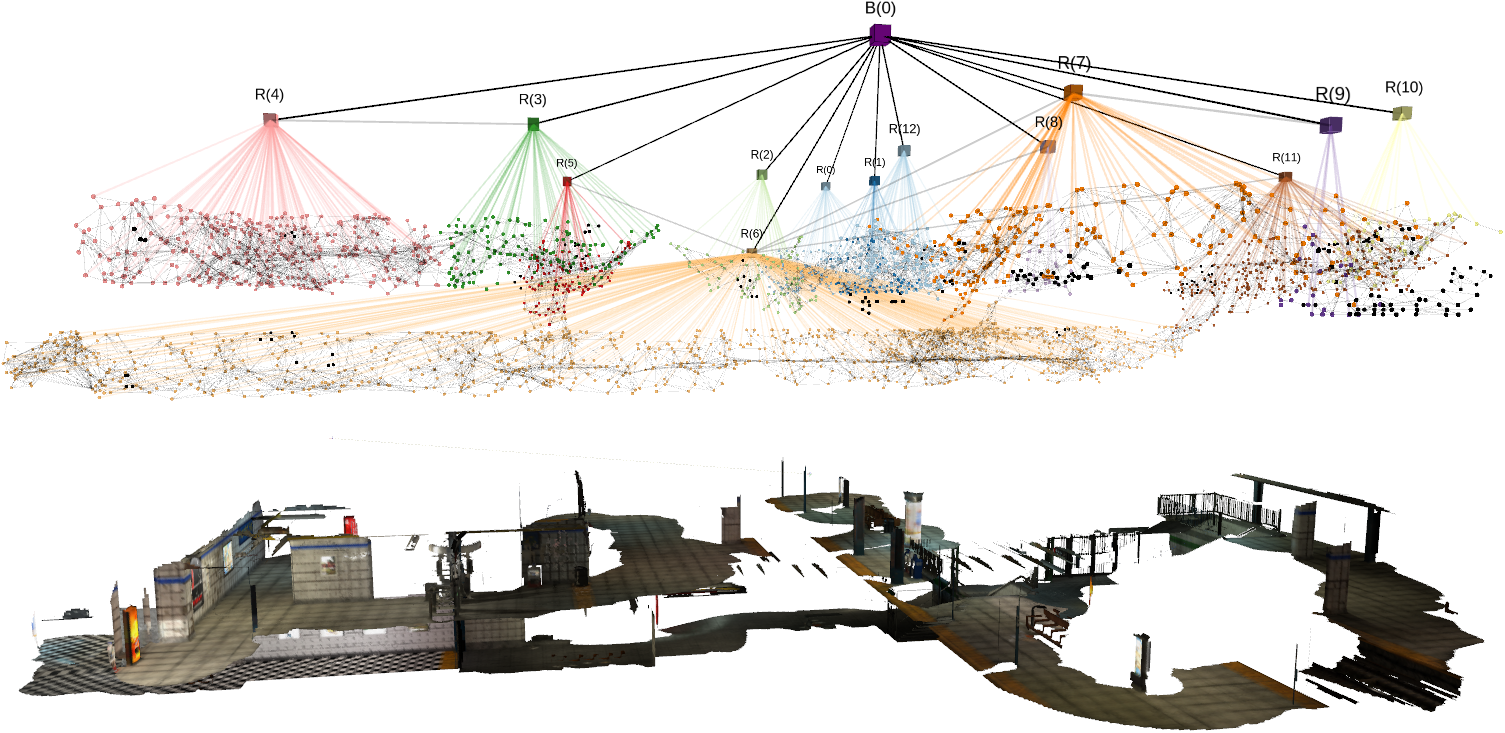}
         \caption{}
         \label{fig:subway_3dsg}
     \end{subfigure}
     
     \begin{subfigure}[b]{0.49\linewidth}
         \centering
         \includegraphics[width=\textwidth]{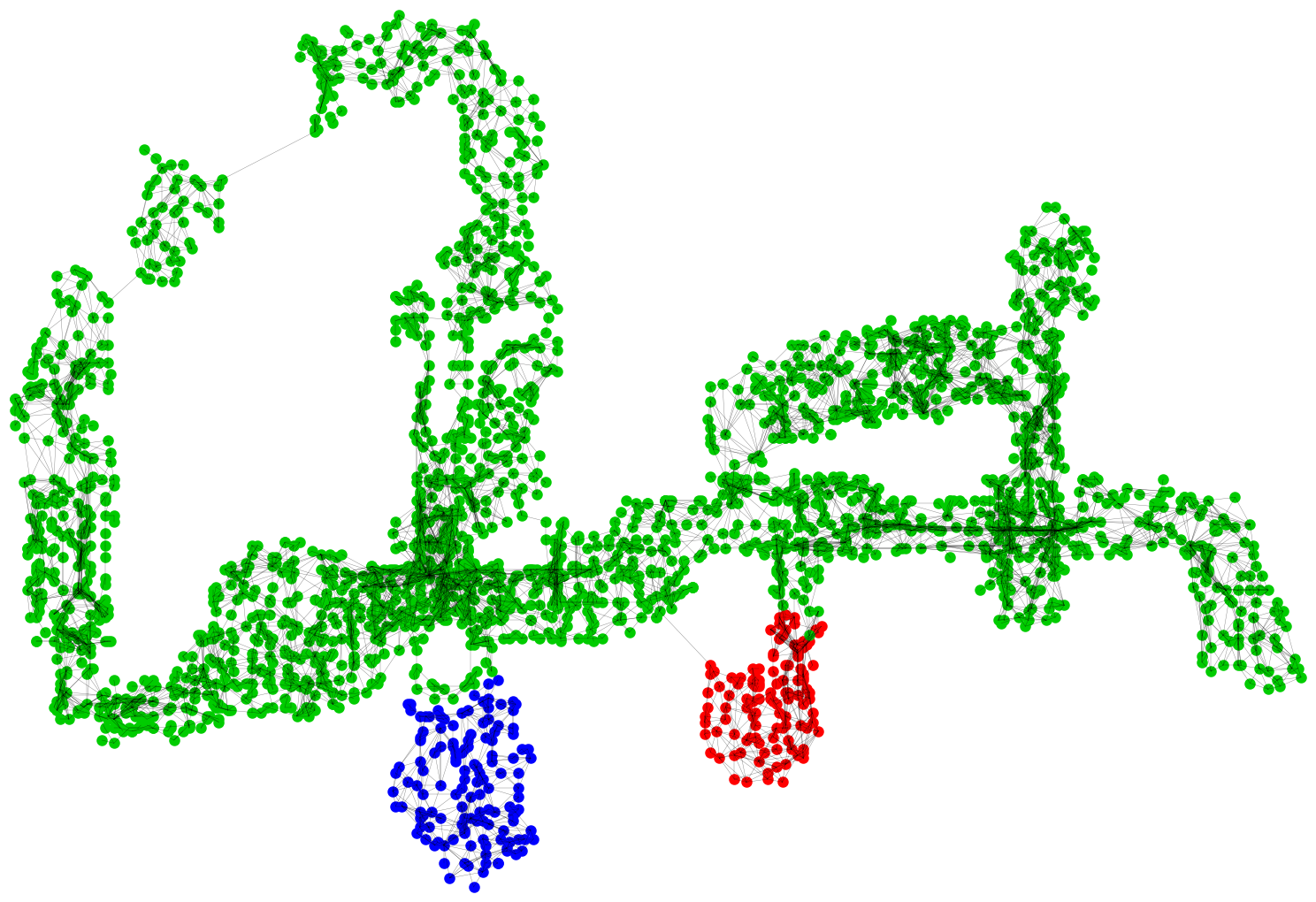}
         \caption{}
         \label{fig:subway_places}
     \end{subfigure}
     \begin{subfigure}[b]{0.49\linewidth}
         \centering
         \includegraphics[width=\textwidth]{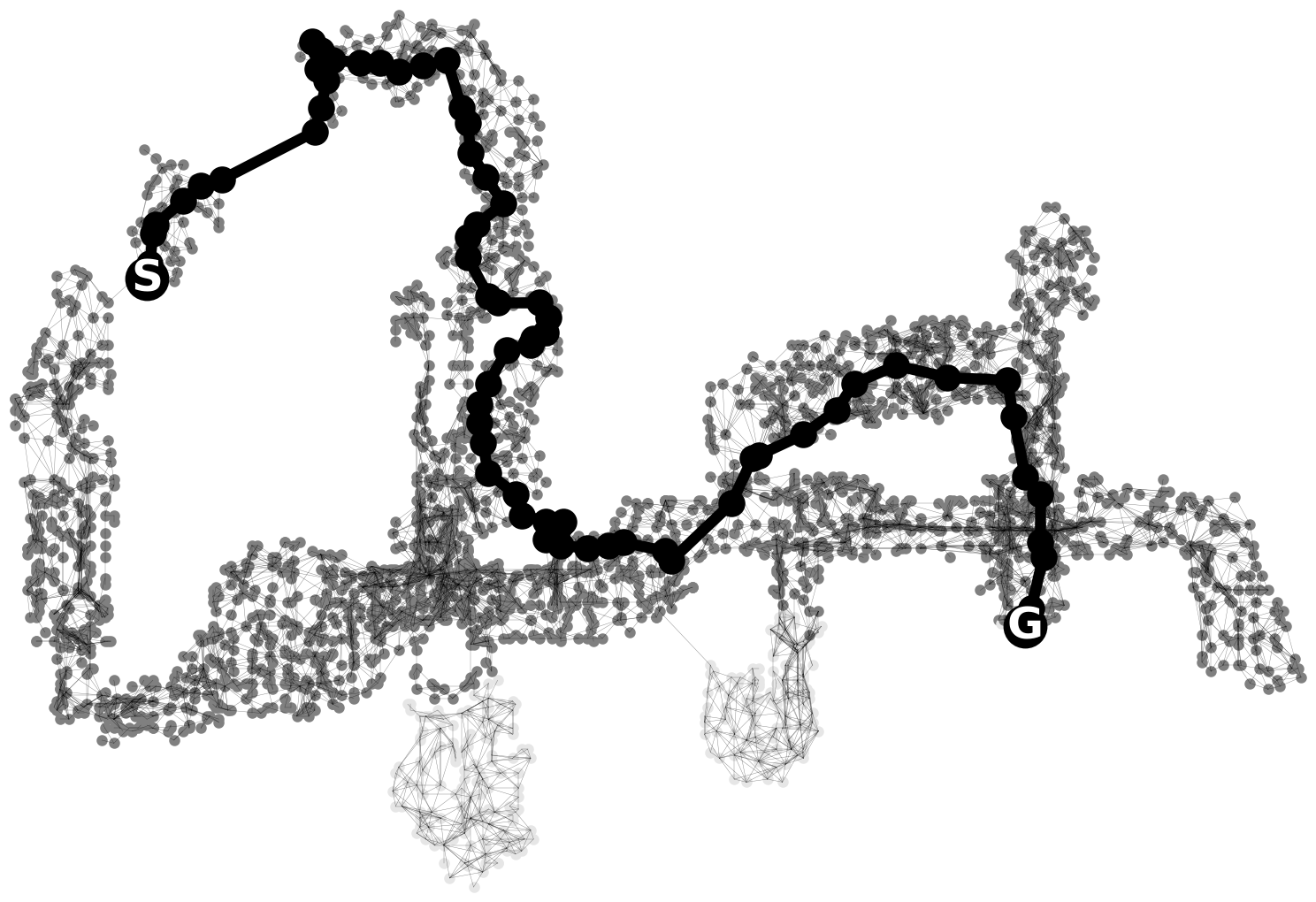}
         \caption{}
         \label{fig:subway_COAStar_places}
     \end{subfigure}
     
     \begin{subfigure}[b]{0.49\linewidth}
         \centering
         \includegraphics[width=\textwidth]{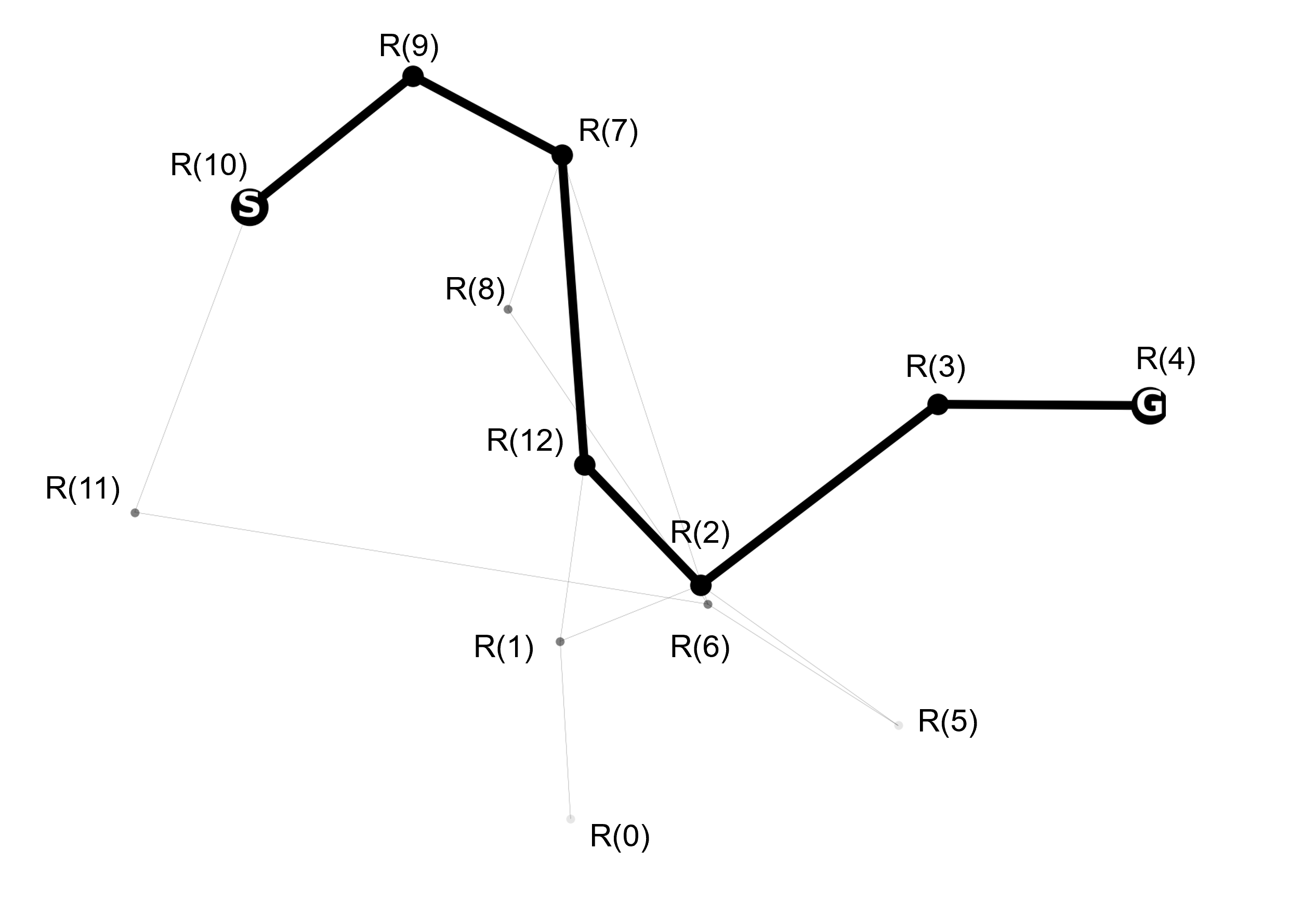}
         \caption{}
         \label{fig:subway_HCOAStar_rooms}
     \end{subfigure}
     \begin{subfigure}[b]{0.49\linewidth}
         \centering
         \includegraphics[width=\textwidth]{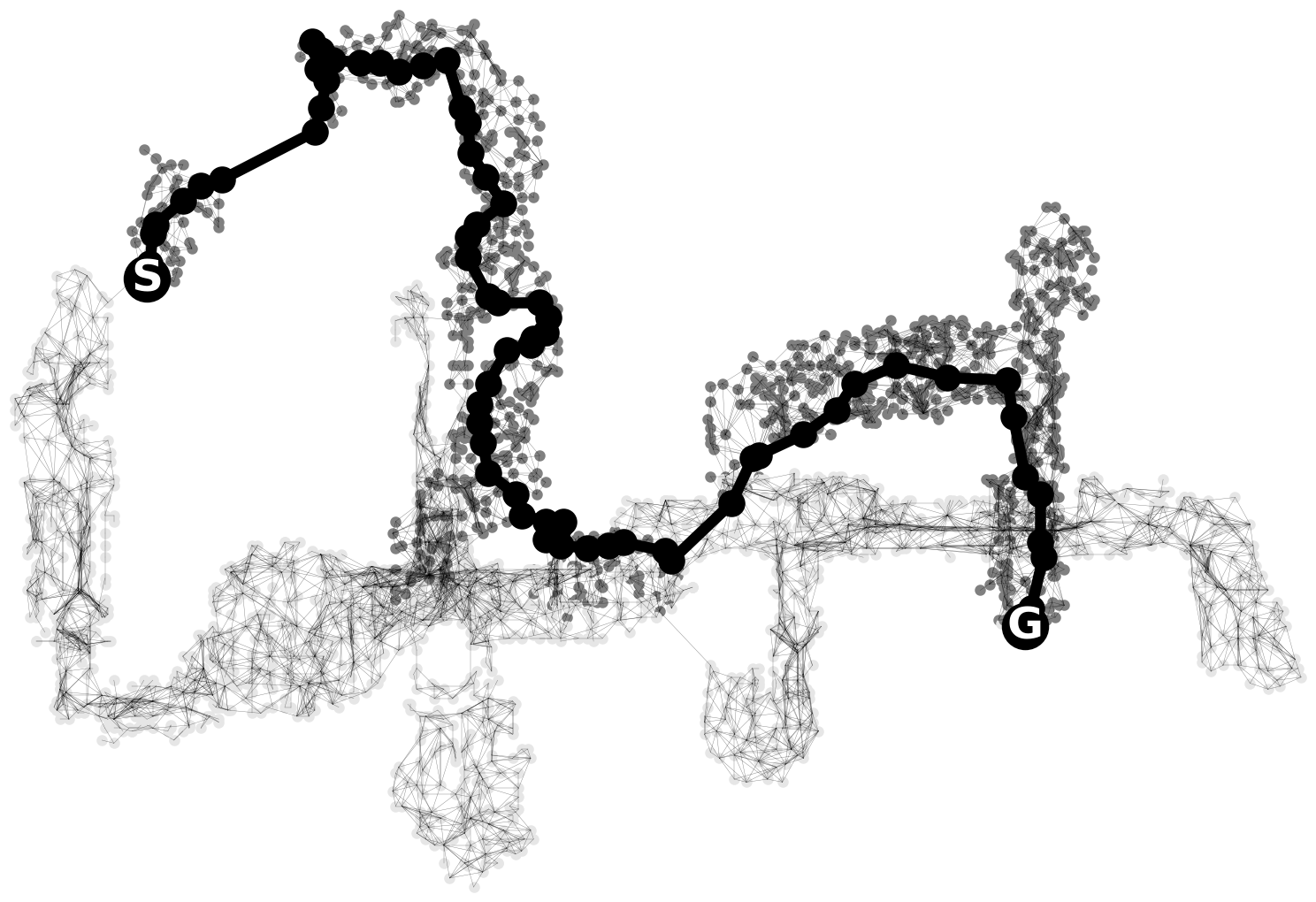}
         \caption{}
         \label{fig:subway_HCOAStar_places}
     \end{subfigure}

     \caption{(a) uHumans2 subway scene; (b) \textsf{Place} subgraph of the uHumans2 subway (88m×60m) with its semantic classes (green, blue, red); (b-d) The starting node is denoted by \textbf{S}, and the goal node by \textbf{G}. The path of each algorithm is shown in black. Expanded nodes are depicted in dark gray, while unexpanded nodes in light gray: (b) Path of COA* on the \textsf{place} subgraph; (c) Path of HCOA* on the \textsf{room} subgraph; (d) Path of HCOA* on the \textsf{place} subgraph.} 
     \label{fig:subway}
\end{figure}

We performed similar simulations to those in Section~\ref{sec:path_planning_office} on the two-floor subway 3DSG, which consists of 2,732 \textsf{place} nodes, shown in Figure~\ref{fig:subway}(\subref{fig:subway_3dsg}). 
The safety constraints are prioritized in descending order: (i) avoid room R(5), where stair repairs are in progress, and (ii) avoid room R(0), which is crowded.
Accordingly, we assign semantic classes as follows: $\phi_V^0(v) = 3$ for all $v \in \text{R(5)}$, $\phi_V^0(v) = 2$ for all $v \in \text{R(0)}$, and $\phi_V^0(v) = 1$ for all other nodes.

To evaluate the computational efficiency of the algorithms, we performed 1,000 runs with fixed start and goal locations. 
The starting node $v_s$ is located in room R(10), and the goal node $v_g$ is in room R(4). 
Figure~\ref{fig:subway}(\subref{fig:subway_places}) shows the \textsf{place} layer, while Figures~\ref{fig:office}(\subref{fig:office_COAStar_places})-(\subref{fig:office_HCOAStar_places}) depict the computed paths of COA* and HCOA*, along with the corresponding expanded nodes.

Table~\ref{tab: results_subway} summarizes the results. 
HCOA* achieves a 59\% reduction in the number of expanded nodes compared to COA*. 
In terms of computational time, HCOA*-MC and HCOA*-kNN yield reductions of 49\% and 15\%, respectively.

We also assessed the suboptimality of the proposed algorithm by running 500 scenarios with varied start and goal locations. 
HCOA* successfully finds the optimal path in 70.56\% of the cases. 
This decrease compared to the office environment is attributed to the increased complexity introduced by the two-floor layout and the irregular geometry of the rooms.
For comparison, the results of MA* show that with $\alpha = 2$, the performance is 72.42\%, while $\alpha = 10$ yields 99.16\%. 
For larger values, accuracy approaches 100\%, but degrades for $\alpha > 10^5$ 
again indicating numerical instability.

\begin{table}[!t]
\caption{Path-Planning on uHumans2 Subway Scene.}
\label{tab: results_subway}
\centering
\small
\begin{tabular}{ |c|c|c|c| }
    \hline
    Algorithm & Exp. Nodes & Time ($10^{-3}$s) & Opt. Path \\
    \hline
    HCOA*-MC & \textbf{1029} & \textbf{9.3} $\boldsymbol{\pm}$ \textbf{4.5} & \checkmark \\
    \hline
    HCOA*-kNN & \textbf{1029} & 15.6 $\pm$ 5.1 & \checkmark \\
    \hline
    HCOA*-GNN & \textbf{1029} & 36.6 $\pm$ 4.4 & \checkmark \\
    \hline
    COA* & 2480 & 18.3 $\pm$ 4.5 & \checkmark \\
    \hline
\end{tabular}
\end{table}

\section{Conclusion}
In this paper, we have addressed the problem of robot navigation in 3D geometric/semantic environments by leveraging the environment's hierarchy for efficient path-planning. 
We have introduced Hierarchical Class-ordered A* (HCOA*), an algorithm that exploits a total order over semantic classes to guide the search process while significantly reducing computational effort. 
To classify higher-layer nodes, we proposed three methods:  a Majority-Class method, a k-Nearest Neighbors method, and a Graph Neural Network method. 
Through simulations on two 3D Scene Graphs, we showed that HCOA* effectively reduces computational cost compared to other state-of-the-art methods. 
Specifically, our results show that HCOA* reduces the computational time of navigation by up to 50\%, while maintaining near-optimal performance.

Future work will address the limitations of using a total order over semantic classes for navigation. 
While total ordering offers a simple and interpretable mechanism for prioritizing semantic categories, it is insufficient for capturing more complex or temporally dependent constraints (e.g., avoid kitchen after visiting bathroom). 
To handle such cases, we plan to incorporate temporal logic~\cite{Kress-Gazit2009_LTL} and explore partial ordering schemes.
These enhancements will also necessitate extending our current path-planning framework to support online replanning capabilities. 
To this end, we aim to develop an online path-planning framework based on HCOA*, enabling adaptive decision-making in real-time, dynamic scenarios.

\bibliographystyle{IEEEtran}
\bibliography{IEEEabrv,refs}

\begin{thebibliography}{10}
\providecommand{\url}[1]{#1}
\csname url@samestyle\endcsname
\providecommand{\newblock}{\relax}
\providecommand{\bibinfo}[2]{#2}
\providecommand{\BIBentrySTDinterwordspacing}{\spaceskip=0pt\relax}
\providecommand{\BIBentryALTinterwordstretchfactor}{4}
\providecommand{\BIBentryALTinterwordspacing}{\spaceskip=\fontdimen2\font plus
\BIBentryALTinterwordstretchfactor\fontdimen3\font minus \fontdimen4\font\relax}
\providecommand{\BIBforeignlanguage}[2]{{%
\expandafter\ifx\csname l@#1\endcsname\relax
\typeout{** WARNING: IEEEtran.bst: No hyphenation pattern has been}%
\typeout{** loaded for the language `#1'. Using the pattern for}%
\typeout{** the default language instead.}%
\else
\language=\csname l@#1\endcsname
\fi
#2}}
\providecommand{\BIBdecl}{\relax}
\BIBdecl

\bibitem{armeni_iccv19}
I.~Armeni, Z.-Y. He, J.~Gwak, A.~R. Zamir, M.~Fischer, J.~Malik, and S.~Savarese, ``3{D} {S}cene {G}raph: A structure for unified semantics, 3{D} space, and camera,'' in \emph{IEEE International Conference on Computer Vision (ICCV)}, Seoul, Korea, Oct 27 - Nov 02 2019, pp. 5663--5672.

\bibitem{hughes2022hydra}
N.~Hughes, Y.~Chang, and L.~Carlone, ``Hydra: A real-time spatial perception system for {3D} scene graph construction and optimization,'' in \emph{Robotics: Science and Systems (RSS)}, New York, NY, June 27 – July 1 2022.

\bibitem{hughes2024foundations}
N.~Hughes, Y.~Chang, S.~Hu, R.~Talak, R.~Abdulhai, J.~Strader, and L.~Carlone, ``Foundations of spatial perception for robotics: Hierarchical representations and real-time systems,'' \emph{The International Journal of Robotics Research}, vol.~43, no.~10, pp. 1457--1505, Feb 2024.

\bibitem{botea2004}
A.~Botea, M.~Müller, and J.~Schaeffer, ``Near optimal hierarchical path-finding ({HPA*}),'' \emph{Journal of Game Development}, vol.~1, pp. 1--30, Jan 2004.

\bibitem{Fernandez1998}
J.~Fernandez and J.~Gonzalez, ``Hierarchical graph search for mobile robot path planning,'' in \emph{IEEE International Conference on Robotics and Automation (ICRA)}, vol.~1, Leuven, Belgium, May 16-21 1998, pp. 656--661.

\bibitem{Warren1993}
C.~W. Warren, ``Fast path planning using modified {A*} method,'' in \emph{IEEE International Conference on Robotics and Automation (ICRA)}, vol.~2, Atlanta, GA, May 1993, pp. 662--667.

\bibitem{kremer2023snav}
P.~Kremer, H.~Bavle, J.~L. Sanchez-Lopez, and H.~Voos, ``{S-Nav}: Semantic-geometric planning for mobile robots,'' 2023, arXiv:2307.01613.

\bibitem{Hart1968_Astar}
P.~E. Hart, N.~J. Nilsson, and B.~Raphael, ``A formal basis for the heuristic determination of minimum cost paths,'' \emph{IEEE Transactions on Systems Science and Cybernetics}, vol.~4, no.~2, pp. 100--107, 1968.

\bibitem{serdel2023_smana}
Q.~Serdel, J.~Marzat, and J.~Moras, ``{SMaNa}: Semantic mapping and navigation architecture for autonomous robots,'' in \emph{International Conference on Informatics in Control, Automation and Robotics (ICINCO)}, Rome, Italy, Nov 13 - 15 2023, pp. 453--464.

\bibitem{Wooden2006}
D.~Wooden and M.~Egerstedt, ``On finding globally optimal paths through weighted colored graphs,'' in \emph{IEEE Conference on Decision and Control (CDC)}, San Diego, CA, Dec 13-15 2006, pp. 1948--1953.

\bibitem{Lim2020}
J.~Lim and P.~Tsiotras, ``A generalized {A*} algorithm for finding globally optimal paths in weighted colored graphs,'' Dec 2020, arXiv:2012.13057.

\bibitem{Lim2021COLPA*}
J.~Lim, O.~Salzman, and P.~Tsiotras, ``Class-ordered {LPA*}: An incremental-search algorithm for weighted colored graphs,'' in \emph{IEEE/RSJ International Conference on Intelligent Robots and Systems (IROS)}, Prague, Czech Republic, Sep 27 - Oct 01 2021, pp. 6907--6913.

\bibitem{KOENIG200493}
S.~Koenig, M.~Likhachev, and D.~Furcy, ``{Lifelong Planning A*},'' \emph{Artificial Intelligence}, vol. 155, no.~1, pp. 93--146, 2004.

\bibitem{Rosinol21ijrr-Kimera}
A.~Rosinol, A.~Violette, M.~Abate, N.~Hughes, Y.~Chang, J.~Shi, A.~Gupta, and L.~Carlone, ``Kimera: from {SLAM} to spatial perception with {3D} dynamic {S}cene {G}raphs,'' \emph{Intl. J. of Robotics Research}, vol.~40, no. 12-14, pp. 1510--1546, 2021.

\bibitem{Ngom2024_Intellimove}
F.~Ngom, H.~Y. Zhang, L.~Zhang, K.~Godary-Dejean, and M.~Huchard, ``{IntelliMove}: Enhancing robotic planning with semantic mapping,'' in \emph{Towards Autonomous Robotic Systems (TAROS)}, London, UK, Aug 21 - 23 2024, pp. 72--83.

\bibitem{ray2024tamp}
A.~Ray, C.~Bradley, L.~Carlone, and N.~Roy, ``Task and motion planning in hierarchical {3D} {S}cene {G}raphs,'' 2024, arXiv:2403.08094.

\bibitem{Kurenkov2020SemanticAG}
A.~Kurenkov, R.~Mart'in-Mart'in, J.~Ichnowski, K.~Goldberg, and S.~Savarese, ``Semantic and geometric modeling with neural message passing in 3d scene graphs for hierarchical mechanical search,'' \emph{2021 IEEE International Conference on Robotics and Automation (ICRA)}, pp. 11\,227--11\,233, May 31 - June 04 2020.

\bibitem{Talak2021NeuralTree}
R.~Talak, S.~Hu, L.~Peng, and L.~Carlone, ``Neural trees for learning on graphs,'' in \emph{Conference on Neural Information Processing Systems (NeurIPS)}, Dec 06-14 2021.

\bibitem{Cover1967_kNN}
T.~Cover and P.~Hart, ``Nearest neighbor pattern classification,'' \emph{IEEE Transactions on Information Theory}, vol.~13, no.~1, pp. 21--27, 1967.

\bibitem{kipf2017semi}
T.~N. Kipf and M.~Welling, ``Semi-supervised classification with graph convolutional networks,'' in \emph{International Conference on Learning Representations (ICLR)}, Toulon, France, April 24-26 2017.

\bibitem{kingma2015adam}
D.~P. Kingma and J.~Ba, ``Adam: A method for stochastic optimization,'' in \emph{International Conference on Learning Representations (ICLR)}, San Diego, CA, USA, May 07 - 09 2015.

\bibitem{Kress-Gazit2009_LTL}
H.~Kress-Gazit, G.~E. Fainekos, and G.~J. Pappas, ``Temporal-logic-based reactive mission and motion planning,'' \emph{IEEE Transactions on Robotics}, vol.~25, no.~6, pp. 1370--1381, 2009.

\end{thebibliography}

\end{document}